\documentclass[twoside,letterpaper]{article}

\PassOptionsToPackage{authoryear,round}{natbib}
\usepackage[hidelinks]{hyperref}

\usepackage[preprint]{aistats2027}

\usepackage[T1]{fontenc}
\usepackage{lmodern} 
\usepackage{booktabs}
\usepackage{nicefrac}
\usepackage{microtype}
\usepackage{xcolor}
\usepackage{color}
\usepackage[most]{tcolorbox}

\usepackage{amsthm}
\usepackage{amssymb}
\usepackage{pifont}
\usepackage{enumitem}
\usepackage{comment}
\usepackage{bm}
\usepackage{bbm}

\usepackage{graphicx}
\usepackage{subcaption}
\usepackage{algorithm}
\usepackage{algorithmic}

\usepackage[disable]{todonotes}

\usepackage[round]{natbib}

\usepackage{hyperref}

\renewcommand{\epsilon}{\varepsilon}
\newcommand{\bb}[1]{\mathbb{#1}}
\newcommand{\rr}{\mathbb{R}}
\newcommand{\nn}{\mathbb{N}}
\newcommand{\pp}[1]{\mathbb{P}\left[#1\right]} 
\newcommand{\ee}[1]{\mathbb{E}\left[#1\right]} 
\newcommand{\var}[1]{\mathbb{V}\text{ar}\left[#1\right]} 
\newcommand{\Sk}{\mathcal{S}_k} 
\newcommand{\ps}{(\Omega, \mathcal{F}, \mathbb{P})} 
\newcommand{\alg}{\mathcal{A} = (T, \hat{\sigma})} 
\newcommand{\pmin}{P_{\textup{min}}} 
\newcommand{\supp}[1]{\textup{Supp}(#1)} 
\newcommand{\kl}[2]{\textup{D}_{\textup{KL}}\left(#1 \, ||\, #2\right)} 

\newcommand{\sigu}{\sigma_u} 
\newcommand{\ut}{(U^t)_{t\in \mathbb{N}}} 
\newcommand{\rum}{\text{RUM}(\bm{u},\nu)} 
\newcommand{\rumv}{\text{RUM}(\bm{v}, \nu)} 
\newcommand{\fr}{\varphi_{FR}} 
\newcommand{\wo}{\varphi_{WO}} 
\newcommand{\ff}{\mathbb{F} = (\mathcal{F}_t)_{t\in \mathbb{N}}} 
\newcommand{\utilitygap}[2]{\Delta_{#1, #2}} 

\newcommand{\pij}{P_{ij}} 
\newcommand{\pji}{P_{ji}} 
\newcommand{\estij}[2]{P_{#1}^{#2}} 
\newcommand{\xij}[2]{\mathbbm{1}_{E_{#1}^{#2}}} 
\newcommand{\svij}[2]{V_{#1}^{#2}} 

\newcommand{\wpi}{P_i} 
\newcommand{\wpj}{P_j} 
\newcommand{\esti}[2]{P_{#1}^{#2}} 
\newcommand{\yi}[2]{\mathbbm{1}_{E_{#1}^{#2}}} 
\newcommand{\svi}[2]{V_{#1}^{#2}} 

\newcommand{\shat}{\hat{\sigma}} 
\newcommand{\shatstar}{\hat{\sigma}^\star} 
\newcommand{\tstar}{T^\star} 
\newcommand{\etastar}{\eta^\star} 
\newcommand{\epdel}{(\varepsilon, \delta)} 
\newcommand{\deltan}[1]{\delta_{#1}} 
\newcommand{\dedge}[2]{#1 \longrightarrow #2} 

\newcommand{\cmark}{\ding{51}} 

\theoremstyle{plain} 
\newtheorem{theorem}{Theorem}

\newtheorem{lemma}[theorem]{Lemma}

\theoremstyle{definition} 
\newtheorem{definition}{Definition}

\newlist{assumptionV}{enumerate}{1}  
\setlist[assumptionV,1]{
    label=\textbf{(V)},
    ref=(V),
    leftmargin=*,
    labelsep=.5em,
    topsep=2pt,
    itemsep=0pt
}
\newlist{assumptionWO}{enumerate}{1} 
\setlist[assumptionWO,1]{
    label=\textbf{(WO\arabic*)},
    ref=(WO\arabic*),
    leftmargin=*,
    labelsep=.5em,
    topsep=2pt,
    itemsep=0pt
}

\begin{document}

\runningtitle{Learning a Ranking from Human Feedback in Log-Concave Random Utility Models}

\twocolumn[
\aistatstitle{Learning a Ranking from Human Feedback \\ in Log-Concave Random Utility Models}
\aistatsauthor{Diego Alovisetti \And Marco Mussi \And Alberto Maria Metelli}
\aistatsaddress{Politecnico di Milano \And Politecnico di Milano \And Politecnico di Milano}
]

\begin{abstract}
We study the problem of recovering the ranking of a fixed set of items according to their unknown numerical utilities. At each interaction with the environment, a learner presents the item set to a human and receives comparative feedback of two types. Under \textit{full-ranking} feedback, each interaction reveals a noisy ranking of all items, whereas under \textit{winner-only} feedback, it reveals only the item ranked first. In both settings, we model human feedback using a \textit{random utility model} with log-concave noise and study the number of observations needed to recover an $\epsilon$-\textit{accurate} ranking with high probability. This novel criterion tolerates ordering errors only between items whose utilities differ by less than $\epsilon$. For both feedback types, we establish worst-case sample-complexity lower bounds and develop algorithms that match these bounds up to logarithmic factors. Neither algorithm requires knowledge of the noise distribution, while only requiring an upper bound on its variance. Our results show that the ranking problem under winner-only feedback is intrinsically harder by exposing the sample complexity dependence on the minimum winning probability across the item set.   
\end{abstract}

\section{Introduction}\label{sec:introduction}
In many machine learning problems, an algorithm deals with a finite set of objects, each associated with a numerical score representing its quality. For example, recommendation systems are concerned with products whose scores reflect users' preferences \citep{hajek_minimaxoptimalinference}. Similarly, in Reinforcement Learning from Human Feedback (RLHF) \citep{christiano_deep}, an agent can take multiple actions whose scores are reward values. In both cases, although every object is assigned a scalar quantity, the latter typically remains latent. Indeed, since the algorithm interacts directly with a human evaluator, such as the user in recommendation systems or the problem designer in RLHF, scalar feedback is either unavailable or unreliable. As the psychometric literature suggests, humans struggle to assign numerical scores, while they can reliably provide ordinal feedback. Consequently, the algorithm must address the set of numerical scores while relying on non-scalar feedback signals. In this work, we formalize this to study a ranking-recovery learning problem. We consider a finite set of items associated with numerical scores that we call \emph{utilities} and a learning algorithm that has no prior knowledge about them and aims to recover the ranking induced by the utilities. To do so, the algorithm sequentially interacts with a human evaluator who provides (noisy) feedback in the form of a (partial) ranking of the items. To model human behavior, we use a Random Utility Model (RUM), a standard paradigm for human preferences \citep{thurstone_rum}, which includes the well-studied Plackett-Luce (PL) model \citep{plackett_pl, luce_pl}. Within this framework, we assume that the latent utilities are perturbed by i.i.d. noise to obtain a random utility vector. Then, a \emph{feedback filter} transforms this vector into the observed feedback signal. We make the additional assumption that the noise distribution is log-concave (e.g., Gaussian, logistic), which is standard in the literature \citep{shah_estimationfrompairwisecomparisons} and ensures several desirable properties \citep{critchlow_probabilitymodelsonrankings}. The feedback filter transforms a vector with scalar entries into a partial ranking. Speficially, we consider two types of filters: 
\begin{itemize}[leftmargin=*,topsep=-1pt,noitemsep]
    \item \textbf{Full-Ranking} (FR), where the complete ordering of the random utilities is observed. 
    \item \textbf{Winner-Only} (WO), in which the learner observes the items with the highest sampled utility.
\end{itemize}
We evaluate the ranking recovery problem within the Probably Approximately Correct (PAC) framework \citep{valiant_pac}, where the approximately correct part is formalized using the novel notion of $\epsilon$-accuracy. The latter, which measures the distance between the recovered ranking and the target, is fundamentally different from standard discrete distances such as Kendall tau \citep{kendall_kendalltaudistance} or maximum rank \citep{busa_kendalltaumaxrank}, in that these are insensitive to the underlying utilities. In fact, our metric effectively penalizes misranking two items depending on their utility gap. By considering the proposed ranking problem within the PAC framework, we aim to address the following research question: 
\begin{tcolorbox}[
    colback=gray!5!white,
    colframe=gray!75!black,
    before skip=6pt,
    after skip=6pt
]
What is the minimum number of \emph{ordinal} feedback signals
needed for the learner to recover the ranking induced by
the items' \emph{scalar} utilities?
\end{tcolorbox}
By answering this question for two common human-generated feedback types, our goal is to quantify their intrinsic learning difficulties. This will result in a quantitative understanding of the difference between using full-ranking rather than winner-only feedback. We stress that the tension between the learning goal, rooted in scalar utilities, and the ordinal feedback used to reach it is precisely what makes the proposed ranking problem non-trivial. The notion of $\epsilon$-accuracy is the very element that allows this tension to be resolved.  Our proposed solution to the ranking problem is the Variance-aware Utility Ranker (VUR), a unified framework comprising two algorithms for recovering the target ranking under full-ranking and winner-only signals, respectively. In summary, our main contributions are:
\begin{itemize}[leftmargin=*,topsep=-1pt,noitemsep]
    \item \textbf{Novel Distance Metric}. We introduce the notion of $\varepsilon$-accuracy through Definition \ref{def: epsilon-accuracy}, which generalizes the $\epsilon$-optimality metric used in best-item identification \citep{saha_bestitem} to learning a full utility ranking. 
    \item \textbf{Unified Framework for Diverse Feedback}. We design two PAC algorithms relying on a common design principle, which allows us to view them as part of a single, unified framework for learning the utility ranking across diverse feedback types. 
    \item \textbf{Distribution-agnostic Algorithms}. We propose two algorithms requiring no exact knowledge of the noise distribution. Instead, they rely solely on the log-concavity assumption plus a known upper bound on the noise variance.  
    \item \textbf{Fundamental Limits of Learnability}. We establish information-theoretic lower bounds for the ranking problem under log-concave RUMs, demonstrating that our algorithms are order-optimal. Crucially, Theorem \ref{thm: sample complexity lower bound WO} identifies the minimum winning probability, $\pmin$, as the fundamental informational bottleneck of WO feedback. 
\end{itemize}
The related works are discussed in Appendix \ref{app: related works}.

\section{Preliminaries}\label{sec:preliminaries}
We use standard big-O notation throughout, with $\tilde{\mathcal{O}}$ suppressing logarithmic factors, except for the confidence-dependent term $\log(1/\delta)$. For any $k\in \nn$, we denote $[k] = \left\{1,\dots,k\right\}$. The set of all permutations on $[k]$, i.e., the bijective functions $\sigma:[k]\rightarrow [k]$, is denoted by $\Sk$. We interpret $\sigma\in \Sk$ as the \textit{ranking} that assigns position $\sigma(i)$ to item $i$. The notation $\ps$ refers to a generic probability space, while $\mathbbm{1}_A$ is the indicator function of any event $A\in \mathcal{F}$. The set of all probability measures on $\rr$ is denoted by $\mathcal{P}(\rr)$. If not stated otherwise, the real line is equipped with the Borel sigma-algebra. If $\mu\in \mathcal{P}(\rr)$ is absolutely continuous with respect to the Lebesgue measure $\lambda$, we write $\mu\ll \lambda$. If $f$ is the density of $\mu$, we denote its support by $\supp{f}$. For a non-decreasing function $g:\rr\rightarrow\rr$, the \textit{generalized inverse} is defined by:
\[
g^{-1}(y) = \sup \{x\in \rr\colon g(x)\leq y\}.
\]
When well-defined, the inverse coincides with the generalized inverse; moreover, the generalized inverse is always non-decreasing.

\textbf{Linear extensions and graphs}. A \textit{strict partial order} $<$ on the finite set $\mathcal{S}$ is a homogeneous binary relation that is irreflexive, asymmetric, and transitive. A general result by \citet{szpilrajn_linearextension} states that, given any strict partial order $<$ on $\mathcal{S}$, there exists a \textit{linear extension} $<^*$, i.e., a total order such that $v<^*w$ whenever $v<w$. A generic directed graph is denoted by $G = (\mathcal{V},\mathcal{E})$, where $\mathcal{V}$ is the vertex set and $\mathcal{E}$ is the edge set. A (directed) edge is of the form $\dedge{v}{w}$, where $v$ and $w$ are distinct vertices. A \textit{tournament} is a directed graph with exactly one directed edge between any pair of distinct vertices. When $G$ is acyclic, its \textit{transitive closure} is the strict partial order ``$<$'' on the vertex set satisfying: for any two vertices $v$ and $w$, $v<w$ if and only if there exists a path in $G$ that goes from $v$ to $w$. 

\textbf{Log-concavity}. Given a probability distribution $\nu\in\mathcal{P}(\rr)$, we say that $\nu$ is \emph{log-concave} if:
\[
\nu(\lambda A + (1-\lambda)B) \geq \nu(A)^\lambda\nu(B)^{1-\lambda},
\] 
whenever the sets $A$ and $B$, and $\lambda\in [0,1]$ are chosen so that $\lambda A + (1-\lambda)B\subseteq \rr$ is measurable. A log-concave distribution on $\rr$ is either a Dirac measure or absolutely continuous with respect to the Lebesgue measure \citep{borell_absolutecontinuity}. In this case, its density and cumulative distribution function (CDF) are log-concave \citep{lovasz_geometryoflogconcavefunctions}.

\section{Problem Formulation}\label{sec:problem formulation}
We fix a positive integer $k\geq 2$ and consider a collection of $k\in\nn$ \emph{items}, which we identify with the set $[k]$. Each item $i$ is associated with a scalar utility $u_i\in \rr$, resulting in the utility vector $\bm{u} = (u_1,\dots, u_k)\in\rr^k$. The corresponding \emph{utility ranking} is defined as the unique permutation $\sigu\in \Sk$ satisfying:\footnote{If $u_i = u_j$, the tie is broken by ordering the two items according to their indices.}
\[
\sigu(i) < \sigu(j) \Longleftrightarrow u_i > u_j \vee (u_i = u_j \wedge i < j).
\]
In this paper, we address the problem of approximating the utility ranking based on noisy \textit{feedback signals} generated by a human evaluator. In our formalization, these arise from a RUM with a centered log-concave noise distribution. To avoid degenerate cases, we assume that the latter is absolutely continuous\footnote{Because a log-concave distribution is either a Dirac measure or abs. continuous, this amounts to excluding the case in which the noise has a Dirac distribution. This results in no loss of generality as the support of the noise can be taken arbitrarily narrow.}. Log-concavity is a standard assumption in the literature \citep{shah_estimationfrompairwisecomparisons}, and implies several desirable properties for ranking models \citep{critchlow_probabilitymodelsonrankings}. To ensure learnability, we impose a sufficient bounded-variance condition. For a fixed $V>0$, we let $\mathcal{N}_V$ be the set of all absolutely continuous centered log-concave distributions $\nu$ satisfying:
\begin{assumptionV}
    \item\label{assump: bounded variance} The variance of distribution $\nu$ is bounded above by $V$.
\end{assumptionV}
Intuitively, this can be interpreted as a reliability constraint. While keeping the exact noise distribution unknown, it ensures that the feedback remains a reliable, albeit noisy, reflection of the latent utility relationships. The bound $V$ represents a parameter of the proposed algorithms. For simplicity, we drop the subscript and write $\mathcal{N}$, whose elements are referred to as \textit{noise distributions}. 

\subsection{Feedback-generation Process}

Given $\bm{u}\in\rr^k$ and $\nu\in \mathcal{N}$, we let $N_1,\dots,N_k$ be i.i.d. random variables drawn from $\nu$, and consider the \emph{random utility vector} $U = (U_1,\dots,U_k)$ given by:
\[
U_i = u_i + N_i, \quad \forall\,  i\in [k].
\]
For $t\in\nn$, we let $U^t = (U_1^t,\dots,U_k^t)$ be an independent copy of $U$, and denote the stochastic process $\ut$ by $\rum$, standing for random utility model with utility $\bm{u}$ and noise distribution $\nu$. The feedback signal observed at time $t$ is obtained by applying a \emph{feedback filter} to the utility vector $U^t$. Specifically, we consider two types of filters.  
\begin{itemize}[leftmargin=*,topsep=1pt,noitemsep]
    \item The \textbf{full-ranking} (FR) filter is the function $\fr:\rr^k\rightarrow\Sk$ such that $\fr(\bm{x})$ is the unique permutation $\sigma$ defined by: $\sigma(i) < \sigma(j)$ if and only if $x_i > x_j$ or $x_i = x_j$ and $i<j$. 
    \item The \textbf{winner-only} (WO) filter is the function $\wo:\rr^k\rightarrow [k]$ given by $\wo(\bm{x}):= \min\{i\in [k]: x_i \geq x_j\,  \forall j\neq i\}$.\footnote{The minimum in the definition of $\wo$ is used for tie-breaking. However, when the noise distribution is absolutely continuous with respect to the Lebesgue measure, ties are a zero-probability event.} 
\end{itemize}
We assume that the feedback filter is fixed throughout the learning episode. Thus, given a feedback filter $\varphi\in \{\fr,\wo\}$, we let $\varphi_t = \varphi(U^t)$ for every $t\in\nn$. The \textit{feedback filtration} is defined as $\ff$, where $\mathcal{F}_t$ is the $\sigma$-algebra generated by $\varphi_1,\dots, \varphi_t$. 

\subsection{The PAC Objective}

Oftentimes, estimation of the utility ranking comes as a byproduct of utility estimation \citep{negahban_rankcentrality, soufiani_rankbreaking}. In contrast, our approach focuses on the ranking itself, thus avoiding the need for explicit utility reconstruction. We achieve this by defining the objective in terms of a novel metric for measuring the distance between two rankings, which effectively relates a misranked pair $(i,j)$ to the latent \emph{utility gap} $\Delta_{ij} := u_i - u_j$. 

\begin{definition}[$\varepsilon$-accuracy]\label{def: epsilon-accuracy}
    Let $\bm{u}\in \rr^k$ with utility ranking $\sigu$, and let $\varepsilon>0$. Given $\sigma\in \Sk$, we say that $\sigma$ is an $\varepsilon$\emph{-accurate representation} of $\sigu$ if the following holds for any two distinct $i$ and $j$: 
    \begin{equation}\label{epsilon-accuracy}
        (\sigma(i) - \sigma(j))(\sigu(i) - \sigu(j)) < 0 \implies |\utilitygap{i}{j}| < \varepsilon. 
    \end{equation}
\end{definition}
We remark that the left-hand side of the implication is equivalent to saying that the relative position of items $i$ and $j$ in $\sigma$ is opposite to the one induced by $\sigu$. In \citep{saha_subsetwisepreferences}, a similar notion is used, where the gap between the PL parameters plays the role of the utility gap in $\epsilon$-accuracy. In their work, \citet{falahatgar18a} also propose a metric related to ours; however, their setting is fundamentally different as no utilities are required for the ranking distribution to be defined. We elaborate on these connections in Appendix~\ref{app: related works}.

To define the PAC objective, we first define an \emph{algorithm} as a pair $\alg$, where $T \in \mathbb{N}$ is an almost-surely finite stopping time with respect to the feedback filtration $\mathbb{F}$, and $\shat$ is an $\mathcal{F}_T$-measurable random permutation of $[k]$. 

\begin{definition}[PAC objective]\label{def: PAC objective}
    Let $\mathcal{A} = (T,\shat)$ be an algorithm that learns a ranking based on the sequential observations $\varphi_t$ generated by a fixed unknown $\rum$ with feedback filter $\varphi$. Given $\delta \in (0,1)$ and $\varepsilon>0$, we say that $\mathcal{A}$ is $\epdel$\emph{-PAC} if:
    \[
    \pp{\shat \text{ is an $\varepsilon$-accurate representation of }\sigu} \geq 1 - \delta.
    \]
    If $\mathcal{A}$ is an $\epdel$-PAC algorithm, we define its \emph{sample complexity} as the expected value $\ee{T}$, where the expectation is taken with respect to the probability measure induced by $\rum$. 
\end{definition}

\section{Full-Ranking Feedback}\label{sec: FR learning}
In this section, we study the ranking problem under full-ranking feedback ($\varphi = \varphi_{\textup{FR}}$). Specifically, we introduce an $\epdel$-PAC algorithm named Full-Ranking Variance-aware Utility Ranker (\textsc{FR-VUR}) and derive an upper bound on its sample complexity. Subsequently, we provide an information-theoretic lower bound, showing that \textsc{FR-VUR} achieves an optimal rate, up to logarithmic dependence. The feedback-generating random utility model $\rum$ is fixed throughout, where the noise distribution $\nu$ satisfies Assumption \ref{assump: bounded variance}. 

\subsection{The Pairwise-Preference Probabilities}\label{sec: PP probabilities}

Under FR feedback, every observation can be converted into $\binom{k}{2}$ pairwise-comparison outcomes by \emph{rank-breaking} \citep{soufiani_rankbreaking}. Our algorithm will do so, using these values to estimate, for any two distinct items $i$ and $j$, the \textit{pairwise-preference probability}:
\[
\pij:= \pp{U_i > U_j}. 
\]
The mutual independence of the noise random variables implies that $\pij>1/2$ if and only if $u_i > u_j$ (see Lemma \ref{lem: consistency of PPPs} in Appendix~\ref{app: FR learning}). For every $t\in\nn$, we define the event $E_{ij}^t := \{\varphi_t(i) < \varphi_t(j)\}$, the set of outcomes for which item $i$ is ranked before item $j$ at time $t$. Then, for every integer $n\geq 2$, we consider the unbiased estimator of $\pij$ given by:
\[
\estij{{ij}}{n} := \frac{1}{n}\sum_{t=1}^n \xij{{ij}}{t}.
\]
The (unbiased) \emph{sample variance} is the quantity:
\[
\svij{ij}{n} := \frac{1}{n(n-1)}\sum_{1\leq t_1 < t_2\leq n}(\xij{ij}{t_1} - \xij{ij}{t_2})^2.
\]
For $\delta\in (0,1)$ and $n\in\nn$, we define the \textit{full-ranking time-}$n$ \textit{confidence} as $\deltan{n} := 6\delta/k(k-1)n^2\pi^2$. Following \citep{maurer_empiricalbernstein}, we derive concentration bounds for $\estij{ij}{n}$ and $\svij{ij}{n}$ in Appendix \ref{app: concentration bounds FR} (see Lemma \ref{lem: concentrations for FR}). These include a standard concentration inequality for the sample variance and an empirical Bernstein-type inequality for $\estij{ij}{n}$. 

\subsection{The Full-Ranking Resolution Factor}

A priori, it is not clear whether the pairwise-preference probabilities are the right estimation target to achieve $\epsilon$-accuracy. Below, we show that this is indeed the case by formalizing the intuition that, when $u_i -u_j \geq \epsilon$, the quantity $P_{ij}$ is large enough to reveal the true preference relation between $i$ and $j$. Consequently, by accurate approximating $P_{ij}$, the sign of $u_i - u_j$ can be inferred. To make this precise, we derive a mapping between utility gaps and pairwise-preference probabilities. Let $i,j\in [k]$ be two items with $u_i > u_j$, and let $\xi:[0,\infty) \rightarrow [1/2,1]$ be the function given by:
\begin{equation}\label{xi for FR 2}
    \xi(t) := \int_\rr F(x+t)f(x)\, \mathrm{d}x,  
\end{equation}
where $F$ and $f$ are the CDF and the density of $\nu$, respectively. Then, we write:
\begin{equation}\label{xi for FR}
\pij = \xi(u_i - u_j) = \frac{\eta(u_i - u_j)}{2} + \frac{1}{2}, 
\end{equation}
where we defined the \emph{full-ranking resolution factor} as $\eta(\epsilon) := 2 \xi(\epsilon) - 1$, for $\epsilon > 0$. Intuitively, as Equation~\eqref{xi for FR} suggests, this quantifies how \textit{strongly} an item $i$ with $u_i = u_j + \epsilon$ is preferable to item $j$. In Appendix~\ref{app: lower bound on eta FR}, we show that $\eta(\epsilon) \geq \etastar(\epsilon)$ for every $\epsilon>0$, where:
\begin{equation}\label{eta star for FR}
   \etastar(\epsilon) := \min\left\{\frac{2}{3}, \frac{\epsilon}{\sqrt{6V}}\right\},
\end{equation}
This fact is crucial to guarantee that the \textsc{FR-VUR} algorithm is agnostic to the noise distribution. 

\subsection{The \textsc{FR-VUR} Algorithm}

We fix $\delta\in (0,1)$ and let $\deltan{n}$ be the time-$n$ confidence defined in the previous section. For a fixed accuracy $\epsilon>0$, we write $\etastar$ to denote the lower bound $\etastar(\epsilon)$ given by Equation~\eqref{eta star for FR}. To formally describe \textsc{FR-VUR} and show that it is $\epdel$-PAC, we will refer to it as $\alg$. Pseudocode is given in Appendix \ref{app: VUR pseudocode}.

\textbf{Stopping rule}. The stopping time $T$ is defined as the minimum positive integer $n\in\nn$ such that the following holds for any two distinct $i,j\in [k]$:
\begin{equation}\label{FR stopping condition}
    \sqrt{\frac{2\svij{ij}{n}\log(4/\deltan{n})}{n}} + \frac{7\log(4/\deltan{n})}{3(n-1)} < \frac{\etastar}{4}. 
\end{equation}
To compute the right-hand side, $\mathcal{A}$ uses $\etastar$, so it needs to know the variance upper bound $V$. However, it does not require exact knowledge of the noise distribution. 

Using the concentration properties $\estij{ij}{n}$ and $\svij{ij}{n}$, the stopping condition ensures that, with high probability:
\[
| \estij{ij}{T} - \pij | < \etastar/4. 
\]
We can think of it as a \emph{multiplicative} requirement\footnote{We use the term \emph{multiplicative} bound to denote the fact that the bound itself depends on the quantity to estimate.} by observing that $\max(\estij{ij}{T}, \estij{ji}{T}) \geq 1/2$. This implies:
\[
|\estij{ij}{T} - \pij| = |\estij{ji}{T} - \pji| \leq \max(\estij{ij}{T}, \estij{ji}{T})\cdot \frac{\etastar}{2}. 
\]
To define the output of $\mathcal{A}$, for every $n\in\nn$, we define the \textit{pairwise-preference graph} $G_n = ([k],\mathcal{E}_n)$ as the directed graph where:
\[
\dedge{i}{j}\in \mathcal{E}_n \Longleftrightarrow \estij{ij}{n} > \frac{1}{2} + \frac{\etastar}{4}.
\]
Additionally, we let the \textit{utility graph} be transitive tournament $G_u = ([k], \mathcal{E}_u)$ whose edge set is characterized as follows: $\dedge{i}{j}\in \mathcal{E}_u$ if and only if either $u_i>u_j$, or $u_i=u_j$ and $i<j$. 

\textbf{Output recommendation}. The $\mathcal{F}_T$-measurable random ranking $\shat$ is determined as follows: 
\begin{itemize}[leftmargin=*,topsep=-1pt,noitemsep]
    \item If $G_T$ contains a cycle, choose $\shat \in \Sk$ arbitrarily.
    \item If $G_T$ is a DAG, let $G^*$ be its transitive closure, corresponding to a strict partial order of its vertices. Then, let $\shat\in \Sk$ be any linear extension of this partial order.
\end{itemize}
In particular, when $G_T$ contains no cycle, $\dedge{i}{j} \in \mathcal{E}_T$ implies $\shat(i) < \shat(j)$.

We show that $\alg$ satisfies the PAC property given by Definition \ref{def: PAC objective}. To this end, we first observe that, with probability at least $1-\delta$, $G_T$ is a subgraph of $G_u$. This follows from the concentration properties of $\estij{ij}{n}$ and $\svij{ij}{n}$ and is proven by Lemma~\ref{lem: the output graph is a DAG} in Appendix~\ref{app: pac property FR}. Second, we note that the $\epsilon$-accuracy target is met if, for any two items $i$ and $j$ with $u_i - u_j \geq \epsilon$, we have $\shat(i) < \shat(j)$. By definition, this is the case if $\dedge{i}{j}$ is an edge of $G_T$. As demonstrated by Lemma \ref{lem: PAC guarantee for FR} in Appendix \ref{app: pac property FR}, this occurs with probability at least $1 - \delta$, thus proving that $\mathcal{A}$ is $\epdel$-PAC. 

The next result provides an upper bound on the sample complexity of the proposed algorithm. The full derivation is provided in Appendix~\ref{app: proof of sample complexity upper bound FR}. 

\begin{theorem}[Sample complexity upper bound for \textsc{FR-VUR}]
\label{thm: sample complexity upper bound FR}
Let $\bm{u}\in\rr^k$ and $\nu\in\mathcal{N}$ be such that
$\textup{RUM}(\bm{u}, \nu)$ satisfies Assumption~\emph{\ref{assump: bounded variance}},
and let $\mathcal{A}$ be the \emph{\textsc{FR-VUR}} algorithm.
Define $Q := \max_{i\neq j}\pij(1-\pij)$. Then, for every $\delta\in(0,1)$ and $\epsilon>0$,
$\mathcal{A}$ is $\epdel$-\emph{PAC}, with sample complexity: 
\begin{equation}\label{sample complexity upper bound FR}
\begin{aligned}
\ee{T}=\mathcal{O}\Biggl(\frac{Q}{{\etastar}^2} \log\Bigl(\frac{Qk}{{\etastar}^2\delta}\Bigr) +\frac{1}{\etastar}\log\Bigl(\frac{k}{\etastar\delta}\Bigr)\Biggr).
\end{aligned}
\end{equation}
\end{theorem}

When $Q$ is non-zero, the upper bound is $\tilde{\mathcal{O}}(Q/{\etastar}^2)$. The term $Q$ follows from the empirical Bernstein-type stopping rule, which adapts to the intrinsic variance of the pairwise comparisons and accelerates convergence when outcomes are highly skewed. To illustrate this advantage, consider the extreme case with distinct utilities and a vanishing noise variance, implying $Q = 0$. In this scenario, the sample complexity reduces to $\tilde{\mathcal{O}}(1/\etastar)$. As $V$ goes to zero, thus matching the true variance, $\etastar$ saturates at $2/3$, and this residual term reduces to the purely logarithmic cost $\tilde{\mathcal{O}}(\log(k/\delta))$. In contrast, a Hoeffding-based approach would rely on the global variance upper bound $Q \leq 1/4$. For $\sqrt{V}\geq 3\epsilon/(2\sqrt{6})$, so that $\etastar = \epsilon/(2\sqrt{6V})$, this would result in the worst-case sample complexity:
\begin{equation}\label{worst-case upper bound FR}
    \ee{T} = \mathcal{O}\left(\frac{V}{\epsilon^2}\log\left(\frac{V}{\epsilon^2}\cdot\frac{k}{\delta}\right)\right). 
\end{equation}
For a fixed $V>0$, the right-hand side becomes $\tilde{\mathcal{O}}(V/\epsilon^2)$, where the term $1/\epsilon^2$ reflects the $\epsilon$-accuracy objective. The $V$ scaling captures the intuition that, if the noise variance increases, the observed rankings become less reliable. As we will establish in Theorem \ref{thm: sample complexity lower bound for FR}, the $V/\epsilon^2$ dependence is information-theoretically optimal. 

\subsection{Information-theoretic Lower Bound}\label{subsec: matching lower bound FR}

In this part, we demonstrate an information-theoretic worst-case lower bound matching the upper bound given in Equation~\eqref{worst-case upper bound FR}. Precisely, we address the \emph{small-gap regime} $\epsilon < \sqrt{6V}/(2\pi)$, which is the most interesting case.\footnote{The threshold $\sqrt{6V}/(2\pi)\approx 0.39\sqrt{V}$ is the one under which the Gumbel instances used in the proof satisfy $\epsilon/\beta < 1/2$, where $\beta$ is the scale parameter. Its precise value is immaterial: our results are asymptotic, thus justifying the restriction to the small-gap regime.} To do so, we identify a class of ``hard'' problem instances, defined by a Gumbel noise paired with a utility vector where one utility gap is exactly $\epsilon$. Our proof, provided in Appendix~\ref{app: proof of sample complexity lower bound FR}, employs a change-of-measure argument following \citep{kaufmann_changeofmeasure}. Below, $\mathcal{A}^* = (\tstar,\shatstar)$ is any ranking algorithm based on FR feedback signals. We assume $\mathcal{A}^*$  is $\epdel$-PAC for any $\delta\in (0,1)$.

\begin{theorem}[Lower bound for full-ranking feedback]\label{thm: sample complexity lower bound for FR}
    Let $\delta\in (0,1/4]$. Then, there exists a noise-utility pair $(\nu,\bm{u})$ satisfying Assumption~\emph{\ref{assump: bounded variance}} such that the sample complexity of algorithm $\mathcal{A}^* = (\tstar,\shatstar)$ satisfies:
    \[
    \ee{\tstar} = \Omega\left(\frac{V}{\epsilon^2}\log\left(\frac{1}{\delta}\right)\right),
    \]
    where the expectation is taken with respect to the probability measure induced by $\rum$. 
\end{theorem}
This result establishes the $V/\epsilon^2$ scaling as a fundamental learning cost of the ranking problem. In particular, Theorem \ref{thm: sample complexity lower bound for FR} shows that the exact dependence on the noise variance is linear.  

\section{Winner-Only Feedback}\label{sec: WO learning}
We now study the ranking recovery problem under WO feedback, meaning $\varphi=\wo$. For this setting, we design and analyze yet another learning algorithm that differs from \textsc{FR-VUR} in that it estimates winning probabilities rather than pairwise-preference probabilities. Nonetheless, both algorithms rely on a (feedback-specific) resolution factor to relate the estimated probabilities to underlying utility gaps. To emphasize this continuity, we overwrite the symbols $\eta$, $\eta^\star$, and $\mathcal{A}$. As in the full-ranking case, $U^t$ is the random utility vector drawn at time $t$ from a fixed $\rum$, where $\nu$ satisfies Assumption~\ref{assump: bounded variance}. Additionally, we let $k\geq 3$, as the winner-only setting reduces to full-ranking when only two items are available. 

\subsection{Observability Assumption}

A full-ranking signal provides information about the relative position of any two items, which was exploited by \textsc{FR-VUR} through rank-breaking. In contrast, under winner-only feedback, information about the relative position of two items is obtained only if one of them wins. Consequently, to correctly rank these, the learner must rely on their winning frequency. Theorem~\ref{thm: sample complexity lower bound WO} below quantifies this by providing an information-theoretic lower bound, demonstrating that the sample complexity of any PAC algorithm is inversely proportional to the minimum winning probability. Therefore, to avoid ending in a non-learnable regime, we enforce an additional assumption. To this end, we begin by defining the \textit{utility range}: 
\begin{equation*}
\Delta_u := \max_{i\in [k]} u_i - \min_{j\in [k]} u_j.
\end{equation*}
Then, denoting by $f$ the density of $\nu$, we let $a:= \inf\supp{f}$ and $b:=\sup\supp{f}$, and call $b-a$ the \emph{width} of the noise support. To ensure that every item have a non-zero winning probability, we impose the following condition: 
\begin{assumptionWO}
    \item\label{assump: the support contains a,b WO} The utility range is smaller than the width of the support of $\nu$, i.e., $\Delta_u < b-a$. 
\end{assumptionWO}
Under this assumption, $\pmin:= \min_{i\in [k]}\wpi > 0$, meaning that, as the episode horizon $T$ grows to infinity, every item will eventually win in at least one round with non-zero probability. With this in mind, we can interpret \ref{assump: the support contains a,b WO} as an \emph{observability assumption}. 
We anticipate that, while taking the variance upper bound $V$ as input, our algorithm does not require knowledge of either $a$ or $b$. 

\subsection{The Winner-Only Resolution Factor}

The probability that an item is the winner plays a fundamental role in this setting. Formally, for every $i\in [k]$, we define its \textit{winning probability} as:
\begin{equation}\label{eq: winning probability}
\wpi := \pp{U_i > U_j, \, \forall j\in [k]\setminus\{i\}}.  
\end{equation}
For any two items $i,j\in [k]$ with $u_i \geq u_j$, their winning probabilities are \textit{consistent}, in the sense that $\wpi > \wpj$ if and only if $u_i > u_j$ (see Lemma~\ref{lem: monotonicity of WPs} in Appendix~\ref{app: monotonicity of WPs}). Since the winning probabilities $P_i$ and $P_j$ depend on the magnitude of \textit{all} utility gaps, it is unclear whether they can be used to infer the magnitude of $\Delta_{ij}$. Remarkably, this turns out to be the case in the (generalized) PL model, a RUM with Gumbel noise with scale parameter $s>0$, where:\footnote{This property is typically referred to as Independence of Irrelevant Alternatives (IIA) \citep{ray_iia}.}
\begin{equation}\label{PL relation between utility gap and wp ratio}
u_i - u_j = s\log\left(\frac{\wpi}{\wpj}\right).
\end{equation}
\citet{saha_subsetwisepreferences} exploited this and used estimates of the winning-probability ratios to approximate the ranking induced by the PL parameters. In this work, we generalize their approach to a broader class of RUMs, while addressing the latent utilities instead of the PL parameters. We do so by establishing a functional relation between winning-probability ratios and utility gaps under the sufficient condition of log-concavity. To illustrate this, we consider $i$ and $j$ such that $u_i > u_j$, and observe that: 
\begin{equation}\label{eq: wp difference}
    \wpi - \wpj \geq \wpi \cdot \frac{1 - e^{-(u_i - u_j)/\sqrt{3V}}}{2}.
\end{equation}
This fact is proven in Lemma~\ref{lem: winning-probability difference lower bound} in Appendix~\ref{app: uniform bound on eta WO} using log-concavity. Now we define the function $(1,2)$-valued function $\xi$ by setting:
\begin{equation}\label{eq: xi for WO}
\xi(t) := \frac{2}{1 + e^{-t / \sqrt{3V}}}, \quad \forall \, t\in (0,b-a).
\end{equation}
Assumption~\ref{assump: the support contains a,b WO} ensures $\wpj > 0$, so we can rearrange the terms of Equation~\eqref{eq: wp difference} to obtain $\wpi/\wpj \geq \xi(u_i - u_j)$. Since $\xi$ is strictly increasing, taking its generalized inverse yields:
\begin{equation}\label{relating utilities to the wp ratio}
    u_i - u_j \leq \xi^{-1}\left(\frac{\wpi}{\wpj}\right).
\end{equation}
Following the approach used in Section~\ref{sec: FR learning}, we define the \textit{winner-only resolution factor} $\eta: (0,b-a)\rightarrow (0,1/3)$ as:
\begin{equation}\label{eta for WO}
    \eta(t) := \frac{\xi(t) - 1}{\xi(t) + 1} = \frac{1-e^{-t/\sqrt{3V}}}{3+e^{-t/\sqrt{3V}}}.
\end{equation}
The latter is defined using the known variance bound $V$ and is independent of the particular noise distribution in $\mathcal{N}_V$. For a simpler expression with the same scaling in the $\epsilon\approx 0$ regime, we let:
\begin{equation}\label{eta star for WO}
    \etastar(\epsilon) := \frac{e - 1}{3e + 1} \min\left\{1 , \frac{\epsilon}{\sqrt{3V}}\right\}. 
\end{equation}
Then, $\eta(\epsilon) \geq \etastar(\epsilon)$ for every $\epsilon\in (0,b-a)$, which is proven using standard analytical tools in Lemma \ref{lem: uniform lower bound for eta WO} in Appendix \ref{app: uniform bound on eta WO}. 

\subsection{The Winning-Probability Estimators}\label{sec: WP estimators}

For each item $i\in [k]$ and every time index $t\in\nn$, we define the event $E_i^t := \{\varphi_t = i\}$, consisting of all outcomes for which $i$ is the winner at time $t$. Then, for $n\in\nn$, we let:
\[
\esti{i}{n} := \frac{1}{n}\sum_{t=1}^n \yi{i}{t},
\]
which is an unbiased estimator for the winning probability $\wpi$. For $n\geq 2$, we define the (unbiased) sample variance as:
\[
\svi{i}{n} := \frac{1}{n(n-1)} \sum_{1 \leq t_1 < t_2 \leq n} (\yi{i}{t_1} - \yi{i}{t_2})^2. 
\]
Given a fixed $\delta\in (0,1)$, we tune the \textit{winner-only time-}$n$ \textit{confidence} to be $\deltan{n} := 3\delta / \pi^2 n^2 k$, for all $n\in\nn$. Concentration inequalities for these estimators are given by Lemma \ref{lem: concentrations for WO} in Appendix \ref{app: concentration bounds WO}  based on \citep{maurer_empiricalbernstein}.

\subsection{The \textsc{WO-VUR} Algorithm}

Given $\delta\in (0,1)$ and $\epsilon\in (0,b-a)$ we let $\deltan{n}$ be the winner-only time-$n$ confidence defined in Section~\ref{sec: WP estimators}, while we write $\etastar$ for the lower bound $\etastar(\epsilon)$ given by Equation~\eqref{eta star for WO}. Below, we introduce the $\epdel$-PAC algorithm \textsc{WO-VUR} $\alg$, where the notation overload emphasizes the continuity with the full-ranking setting. This algorithm uses empirical winning frequencies to deliver its final recommendation. The resulting sample-complexity guarantee holds for any noise-utility pair $(\nu,\bm{u})$ such that $\nu\in\mathcal{N}_V$ and assumption~\ref{assump: the support contains a,b WO} holds. Knowledge of $V$ is required, although the algorithm does not require the noise density or its support.  Pseudocode is given in Appendix~\ref{app: VUR pseudocode}. 

\textbf{Stopping rule}. The algorithm stops at time $T$, which is defined as the minimum integer $n\geq 2$ such that the following holds for every $i\in [k]$:
\begin{equation}\label{stopping rule for WO}
    \sqrt{\frac{2\svi{i}{n}\log(4/\deltan{n})}{n}} + \frac{7\log(4/\deltan{n})}{3(n-1)} < \frac{\etastar}{2} \cdot \esti{i}{n}.
\end{equation}

Given the estimates $\esti{1}{n},\dots, \esti{k}{n}$ at time $n$, we let $\sigma_n$ be the unique permutation of $[k]$ satisfying: 
\[
\sigma_n(i) < \sigma_n(j) \Longleftrightarrow \esti{i}{n} > \esti{j}{n} \vee \left(\esti{i}{n} = \esti{j}{n} \, \text{and}\, i<j\right). 
\]

\textbf{Output recommendation}. The recommendation is defined as the $\mathcal{F}_T$-measurable permutation $\shat:= \sigma_T$. 

To show that $\mathcal{A}$ has the PAC property given by Definition~\ref{def: PAC objective}, we suppose that, for every $i\in [k]$:
\[
|\esti{i}{T} - \wpi| \leq \sqrt{\frac{2\svi{i}{T}\log(4/\deltan{T})}{T}} + \frac{7\log(4/\deltan{T})}{3(T-1)}.
\]
As shown by Lemma \ref{lem: concentrations for WO} in Appendix \ref{app: concentration bounds WO}, this is true with probability at least $1 - \delta$. Using the definition of $T$ and the fact that $\etastar < 1$, the following holds for every item $i$: 
\[
|\esti{i}{T} - \wpi| < \esti{i}{T}/2.
\]
This highlights the \emph{multiplicative} nature of the bound. Now, suppose that $i$ and $j$ are misranked at time $T$, meaning $(\wpi - \wpj)(\esti{i}{T} - \esti{j}{T}) \leq 0$. Without loss of generality, we may assume $u_i > u_j$, thus:
\[
0 < \wpi - \wpj < \frac{\etastar}{2}\left(\esti{i}{T} + \esti{j}{T}\right) \leq \etastar\cdot (\wpi + \wpj). 
\]
Rearranging these terms yields:
\[
\frac{\wpi}{\wpj} < \frac{1 + \etastar}{1 - \etastar} \leq \frac{1 + \eta(\epsilon)}{1 - \eta(\epsilon)} = \xi(\epsilon),
\]
where $\eta(\epsilon)$ is the feedback-specific resolution factor defined by Equation~\eqref{eta for WO}. Our earlier calculations give $\xi(u_i - u_j) \leq \wpi/\wpj$, thus implying $\xi(u_i-u_j) < \xi(\epsilon)$. Since $\xi$ is an increasing function, we conclude that $u_i - u_j < \epsilon$. In other words, we have shown that: with probability at least $1 - \delta$, if $i$ and $j$ are misranked, then their absolute utility gap is bounded by $\epsilon$. Hence, algorithm $\mathcal{A}$ is $\epdel$-PAC. 

\subsubsection{Sample Complexity Upper Bound}

The next result derives an upper bound on the sample complexity of the \textsc{WO-VUR} algorithm. Formal derivations are given in Appendix \ref{app: proof of sample complexity upper bound WO}.

\begin{theorem}[Sample complexity upper bound for \textsc{WO-VUR}]\label{thm: sample complexity upper bound for WO}
    Let $\bm{u}\in \rr^k$ and $\nu\in\mathcal{N}_V$ be a noise-utility pair satsfying Assumption~\emph{\ref{assump: the support contains a,b WO}}, and let $\alg$ be the \emph{\textsc{WO-VUR}} algorithm. Then, for every $\delta\in(0,1)$ and $\epsilon\in (0,b-a)$, $\mathcal{A}$ is $\epdel$-PAC with sample complexity:
    \begin{equation}\label{sample complexity upper bound WO}
    \ee{T} = \mathcal{O}\left(\frac{1}{\pmin\,{\etastar}^2}\log\left(\frac{1}{\pmin\,{\etastar}^2}\cdot\frac{k}{\delta}\right)\right).
    \end{equation}
\end{theorem}
Using the explicit formula~\eqref{eta star for WO}, in the small-gap regime $\epsilon\leq \sqrt{3V}$, the sample-complexity upper bound takes the form:
    \begin{equation}\label{explicit epsilon sample complexity upper bound WO}
        \ee{T} = \mathcal{O}\left(\frac{V}{\pmin\,\epsilon^2}\log\left(\frac{V}{\pmin\,\epsilon^2}\cdot\frac{k}{\delta}\right)\right). 
    \end{equation}
As we highlighted through our informal discussion, the term $\pmin$ exposes a fundamental learning limit. Since the relative position of $i$ and $j$ is estimated based on their empirical winning frequencies, if these items rarely win, reliable estimation will require a large number of samples. Crucially, the (inverse) linear dependence on $\pmin$ is a consequence of the Bernstein-type stopping rule, whereas using a Hoeffding-based approach would have resulted in a suboptimal quadratic dependence. We make this explicit through Equation~\eqref{bernstein advantage WO} in Appendix \ref{app: proof of sample complexity upper bound WO}. The $1/\epsilon^2$ scaling given in the upper bound reflects the $\epsilon$-accuracy target and mirrors full-ranking case. Similarly, the $V$ term confirms the linear dependence on the reliability of the feedback signals. Finally, the quantity $\log(k)$ results in a cost of lower order. As we shall see, Theorem \ref{thm: sample complexity lower bound WO} establishes the $V/(\epsilon^2\pmin)$ scaling as information-theoretically optimal.  

\subsection{Information-theoretic Lower Bound}

The subsequent lower-bound analysis ultimately demonstrates that learning becomes unattainable when $\pmin$ goes to zero. We specifically analyze sample complexity over all RUMs with minimum winning probability $\pmin \in (0,1/k)$, excluding the case $\pmin = 1/k$, which corresponds to the trivial instance where all utilities are equal. To formalize our result, we let $\mathcal{A}^* = (\tstar, \shatstar)$ be a PAC algorithm that outputs a ranking based on winner-only feedback. Full derivation of the lower bound is provided in Appendix~\ref{app: proof of sample complexity lower bound WO}. 
\begin{theorem}[Lower bound for winner-only feedback]\label{thm: sample complexity lower bound WO}
    Let $k\geq 3$, $\pmin \in (0,1/k)$, and $\delta\in (0,1/4]$ be fixed. Then, there exists a noise-utility pair $(\nu,\bm{u})$ satisfying Assumptions~\emph{\ref{assump: bounded variance}} and \emph{\ref{assump: the support contains a,b WO}} such that the sample complexity of algorithm $\mathcal{A}^* = (\tstar,\shatstar)$ satisfies:
    \[
    \ee{\tstar} = \Omega\left(\frac{V}{\pmin\,\epsilon^2}\log\left(\frac{1}{\delta}\right)\right),
    \]
    where the expectation is taken with respect to the probability measure induced by $\textup{RUM}(\nu, \bm{u})$. 
\end{theorem} 

We begin by observing that our lower bound contrasts with the $\tilde{\mathcal{O}}(k/\epsilon^2)$ rate claimed by \citet{saha_subsetwisepreferences} under the PL model. The discrepancy is due to a fundamental objective misalignment. Indeed, their framework seeks an $\epsilon$-\textit{optimal} ranking, where misclassification is penalized based on the difference between the PL parameters $e^{u_\ell}$. If $i$ and $j$ have utilities $u_i, u_j \ll 0$, then $e^{u_i}, e^{u_j} \approx 0$. Consequently, under the $\epsilon$-optimality metric, misranking these two is tolerable. Therefore, by choosing to use this metric, the true informational bottleneck of $\pmin$ established in Theorem \ref{thm: sample complexity lower bound WO} is deliberately circumenvented. Hence, the claim of full ranking recovery follows without actually resolving the relative order of items with a low winning probability. In contrast, our analysis effectively exposes the difficulties of learning the utility ranking under winner-only feedback, while treating every item pair as equally important. To do so, we chose the $\epsilon$-accuracy metric to depend on the underlying utilities rather than the PL parameters. 

\section{Open Challenges}\label{sec:open challenges}
We characterized the learning difficulty of the utility-ranking recovery problem under two types of RUM-generated feedback. However, full-ranking and winner-only are the extremes of a wide spectrum. Therefore, understanding how the learning task changes across the whole range of feedback types in between (e.g., top-$m$ \citep{hajek_minimaxoptimalinference}) is an open question. Second, we relied on log-concavity to derive the utility-to-probability mappings used by both our $\textsc{VUR}$ algorithms. Without this assumption, the feedback-specific resolution factors can no longer be defined. Understanding whether the same approach can be employed under a milder condition on the noise distribution is not understood yet. Third, a more complex interaction protocol may be considerd. In our work, we focused on a setting where the feedback signal regards all items at every learner-environment interaction because our goal was to quantify the learning difficulty due to using winner-only vs. full-ranking feedback. However, allowing for subset selection \citep{saha_subsetwisepreferences} would accommodate more realistic types of interaction. Finally, our $\epsilon$-accuracy metric effectively links ranking information to the utilities of the ranked items. We believe that in this lies a potential tool for addressing the regret minimization problem in the context of multi-armed bandits with ranking feedback.



\bibliographystyle{plainnat}
\bibliography{references}

\clearpage
\appendix
\onecolumn
\section{Related Works}\label{app: related works}
Our work relates to the existing literature along three primary dimensions: the way of generating feedback, the definition of the objectives and metrics, and the types of feedback used to achieve the target.

A standard paradigm in preference learning is inferring latent properties from human judgments, such as distinguishing ``good'' from ``bad'' trajectories in RLHF \citep{wirth_pbrlsurvey,christiano_deep} or predicting user choices in recommender systems \citep{mcfadden_probabilisticchoice, baltrunas_recommendationsystems}. While the PL model \citep{plackett_pl, luce_pl} is the ubiquitous choice for these tasks, it is a specific instance of the broader Random Utility Model (RUM) framework \citep{thurstone_rum}. Our work adopts the RUM framework with log-concave noise, which provides a unifying lens for these models. Several works implicitly utilize log-concavity by considering the PL model \citep{saha_subsetwisepreferences, chen_spectralmle, negahban_rankcentrality, khetan_datadrivenrankbreaking, soufiani_rankbreaking}, while others do this explicitly \citep{hajek_minimaxoptimalinference}. Remarkably, log-concavity of the noise ensures several desirable properties of the implied ranking distribution \citep{critchlow_probabilitymodelsonrankings}. In this paper, we exploit this regularity to derive uniform lower bounds on the resolution factors. 

The definition of ``preference recovery'' varies significantly across the literature. Multi-armed and dueling bandit frameworks often focus on identifying a single best item \citep{saha_bestitem, yue_duelingbandits, zoghi_relativeucb}, a paradigm known as best-item identification. Many approaches focus on statistical parameter estimation, with particular attention devoted to the estimation of PL parameters, typically measured by the $L_2$-norm \citep{azari_generalizedmm, hajek_minimaxoptimalinference, shah_estimationfrompairwisecomparisons}. When the goal is a ranking, metrics like Kendall-tau distance \citep{kendall_kendalltaudistance} or maximum rank difference \citep{busa_kendalltaumaxrank} are standard. When the utility ranking is the target, it is possible to bypass utilities entirely by assuming distributions over permutations, such as the Mallows model \citep{lu_learningmallows, busa_mallows}. In contrast to these approaches, we operate within the PAC framework \citep{valiant_pac}, defining success through the notion of $\varepsilon$-accuracy. This objective bridges the gap between pure parameter estimation and exact ordinal recovery, by incorporating utility information into the target ranking in an intuitive manner.  

Finally, we distinguish works by the density of the feedback signal. In line with the present paper, some works consider winner-only feedback signals \citep{saha_subsetwisepreferences, saha_bestitem}, which are generalized through the notion of top-$m$ feedback. For full-ranking feedback, a common technique is rank-breaking, which decomposes permutations into pairwise comparisons modeled via a comparison graph \citep{soufiani_rankbreaking, khetan_datadrivenrankbreaking, negahban_rankcentrality}. 

\textbf{Distance Metrics for Ranking Accuracy}. There exist several ways of measuring the distance between two rankings, including the popular Kendall tau distance and the maximum rank difference \citep{kendall_kendalltaudistance, busa_kendalltaumaxrank}. These do not require the ranked items to be associated with utility scores and penalize misclassification uniformly across all item pairs, thus making the comparison with our notion of $\epsilon$-accuracy vacuous. In contrast, we can provide a more meaningful comparison for the following measures.
\begin{itemize}
    \item $\epsilon$\emph{-optimality}. This was introduced by \citet{saha_bestitem} and is similar to the one given by Definition~\ref{def: epsilon-accuracy}. Indeed, an item is $\epsilon$-optimal if its utility value is smaller than the largest utility by less than $\epsilon$. However, in contrast to our metric, this notion of distance induces a much simpler objective as it is only concerned with determining a nearly optimal item.
    \item $\epsilon$\emph{-best ranking}. This notion was introduced by \citet{saha_subsetwisepreferences} and is intrinsically gap-dependent. However, despite its similarities with our notion of $\epsilon$-accuracy, it does not refer to utility gaps. Instead, it is concerned with the gaps between the Plackett-Luce parameters. The latter correspond, up to constants, to the winning probabilities and can be mapped to the utility values through a non-linear function. Consequently, the objective implied by $\epsilon$-optimality is essentially different from ours. In particular, the former is much more tolerant towards items with a small utility because their winning probabilities are small. Optimizing with respect to this metric obscures the true difficulty of the WO feedback, that is, its sensitivity to $\pmin$.
    \item $\epsilon$\emph{-ordering}. This notion was introduced by \citet{falahatgar18a}. For a fixed noise distribution, using the full-ranking resolution factor $\eta$ introduced in Section~\ref{sec: FR learning}, one can show that: 
    \[
    \epsilon-\text{accurate (ours)} \Longleftrightarrow \frac{\eta(\epsilon)}{2}-\text{ranking (\citet{falahatgar18a})}.
    \]
    This equivalence can only be derived if we choose a ranking model induced by a RUM (thus diverging from the setting of \citep{falahatgar18a}) and use the true resolution factor $\eta$ instead of the lower bound $\etastar$ given by Equation~\ref{eta star for FR}. Most importantly, the notion of $\epsilon$-ordering was given in the context of ranking models where there are no underlying utilities that generate the ranking. 
\end{itemize}

\textbf{Comparison with \citet{saha_bestitem}}. Focusing on winner-only feedback, we elaborate on the differences and similarities between our work and that of \citet{saha_bestitem}. While both frameworks seek to lower-bound the winning-probability ratio $\wpi/\wpj$ in terms of the utility gap $u_i - u_j$, the methodologies and objectives differ in two key aspects. First, \citet{saha_bestitem} focus on identifying the $\varepsilon$-best item, whereas we target the recovery of the full utility ranking. Under WO feedback, the latter becomes a significantly more challenging problem. This is due to the informational bottleneck represented by the minimum winning probability $\pmin$, as established by our lower bound analysis formalized through Theorem \ref{thm: sample complexity lower bound WO}. Second, their approach uses the so-called minimum advantage ratio and relies on a distribution-specific constant $c$, which is derived via a so-called ``variational lower bound''. This constant must be computed for each noise model and provided to the algorithm as a parameter. In contrast, our approach leverages the structural regularity of log-concave distributions. Consequently, our algorithms only require an upper bound on the noise variance, regardless of any other distribution-specific parameters. 

\section{Variance-aware Utility Rankers}\label{app: VUR misc}
\subsection{Distribution Table}

The table below lists several examples of log-concave distributions and, for each of them, provides explicit conditions ensuring that the mean is zero and Assumption \ref{assump: bounded variance} is satisfied. If these two are met, the given distribution is allowed by the \textsc{FR-VUR} algorithm presented in Section~\ref{sec: FR learning}. Provided that the additional utility-range-dependent condition \ref{assump: the support contains a,b WO} is also met, the derived sample-complexity guarantees hold for both \textsc{FR-VUR} and \textsc{WO-VUR}. 

\begin{table}[ht]
    \centering
    \begin{tabular}{@{}l c c c c@{}}
    \toprule 
    \textbf{Noise distribution} & \textbf{Centeredness} & \textbf{Bounded-variance \ref{assump: bounded variance}} & \textbf{\textsc{FR-VUR}} & \textbf{WO-VUR} \\ \midrule 
    Gumbel$(\mu, s)$ & $\mu = -s\gamma$ & $s\leq \sqrt{6V}/\pi$ & \cmark & \cmark \\ 
    Normal$(\mu, \sigma^2)$ & $\mu = 0$  & $\sigma\leq \sqrt{V}$ & \cmark & \cmark \\
    Logistic$(\mu,s)$ & $\mu = 0$ &  $s\leq \sqrt{3V}/\pi$ & \cmark & \cmark \\
    Laplace$(\mu,s)$ & $\mu = 0$ & $s \leq \sqrt{V/2}$ & \cmark & \cmark \\
    Uniform$(c,d)$ & $c = -d$ & $d-c\leq \sqrt{12V}$  & \cmark & $\text{\cmark}^{\dagger}$ \\ \bottomrule
    \end{tabular}
    \vspace{1mm}
    \caption{The parameter $\gamma$ in the centeredness condition of the Gumbel distribution is the Euler-Mascheroni constant, whose value is approximately $0.58$. All five families satisfy the noise assumptions for both algorithms under the stated parameter conditions. For WO feedback, $\dagger$ additionally requires $\Delta_u<d-c$ for uniform noise. The other four families have full support, so Assumption~\ref{assump: the support contains a,b WO} holds automatically for any set of utilities.}
    \label{tab: comparison table}
\end{table}

\subsection{Pseudocode for Variance-aware Utility-Rankers}\label{app: VUR pseudocode}

Below is the pseudocode for the \textsc{VUR} algorithms presented in Sections \ref{sec: FR learning} and \ref{sec: WO learning}, respectively.

\begin{algorithm}
    \caption{Full-Ranking Variance-aware Utility Ranker}\label{alg: FR-VUR}
    \renewcommand{\algorithmicensure}{\textbf{Output:}}
    \renewcommand{\algorithmicrequire}{\textbf{Input:}}
    \begin{algorithmic}[1]
        \REQUIRE 
        Items $[k]$, confidence $\delta\in (0,1)$, accuracy $\varepsilon > 0$, variance bound $V > 0$
        \vspace{0.1cm}
        \STATE \textbf{Initialize}: $n \gets 0$, $G \gets ([k], \emptyset)$, $\etastar \gets \min(\frac{2}{3}, \frac{\varepsilon}{\sqrt{6V}})$
        \STATE \phantom{\textbf{Initialize}:} $\pij \gets 0$, $V_{ij} \gets 0$, and $I_{ij} \gets \etastar/4$ for every $i\neq j$ in $[k]$ \\
        \REPEAT
            \STATE $n \gets n+1$ \\ 
            \STATE Observe full ranking $\sigma_n$ from $\rum$ and set $\deltan{n} \gets \frac{6\delta}{k(k-1)n^2\pi^2}$
            \FOR{each pair $i\neq j$}
                \STATE $\pij \gets \frac{n-1}{n}P_{ij} + \frac{1}{n}\mathbbm{1}\{i \text{ is ranked before } j \text{ in } \sigma_n\}$ \\
                \IF{$n > 1$}
                    \STATE $V_{ij} \gets \frac{n}{n-1}P_{ij}(1-P_{ij})$
                    \STATE $I_{ij} \gets \sqrt{\frac{2V_{ij}\ln(4/\delta_n)}{n}} + \frac{7\ln(4/\delta_n)}{3(n-1)}$
                \ENDIF
            \ENDFOR
        \UNTIL{$I_{ij} < \etastar/4$ for every $i\neq j$}
        \FOR{$i,j$ such that $1\leq i < j\leq k$}
            \IF{$\pij > 1/2 + \etastar/4$}
                \STATE Add $\dedge{i}{j}$ to $G$
            \ELSIF{$\pji > 1/2 + \etastar/4$}
                \STATE Add $\dedge{j}{i}$ to $G$
            \ENDIF
        \ENDFOR
        \IF{$G$ is acyclic}
            \STATE $G \gets$ transitive closure of $G$ and $\hat{\sigma} \gets \text{topological sort of } G$
        \ELSE
            \STATE $\shat \gets$ arbitrary permutation of $[k]$
        \ENDIF
        \ENSURE $\shat$
    \end{algorithmic}
\end{algorithm}

\begin{algorithm}
    \caption{Winner-Only Variance-aware Utility Ranker}\label{alg: WO-VUR}
    \renewcommand{\algorithmicensure}{\textbf{Output:}}
    \renewcommand{\algorithmicrequire}{\textbf{Input:}}
    \begin{algorithmic}[1]
        \REQUIRE 
        Items $[k]$, confidence $\delta\in (0,1)$, accuracy $\varepsilon \in (0, b-a)$, variance bound $V > 0$
        \vspace{0.1cm}
        \STATE \textbf{Initialize}: $n \gets 0$, $\etastar \gets \frac{e-1}{3e+1}\min\left\{1,\frac{\varepsilon}{\sqrt{3V}}\right\}$
        \STATE \phantom{\textbf{Initialize}:} $\wpi \gets 1$, $V_i \gets 1/4$, and $I_i \gets \etastar/2$ for every $i \in [k]$ \\
        \REPEAT
            \STATE $n \gets n+1$ \\ 
            \STATE Observe winner $w_n$ from $\rum$ and set $\deltan{n} \gets \frac{3\delta}{k\pi^2n^2}$
            \FOR{each $i\in [k]$}
                \STATE $\wpi \gets \frac{n-1}{n}\wpi + \frac{1}{n}\mathbbm{1}\{w_n = i\}$ \\
                \IF{$n > 1$}
                    \STATE $V_i \gets \frac{n}{n-1}\wpi(1-\wpi)$
                    \STATE $I_i \gets \sqrt{\frac{2V_i\ln(4/\delta_n)}{n}} + \frac{7\ln(4/\delta_n)}{3(n-1)}$
                \ENDIF
            \ENDFOR
        \UNTIL{$I_i < \wpi\cdot \etastar/2$ for every $i\in [k]$}
        \STATE $\shat \gets$ unique permutation $\sigma\in \Sk$ such that $\sigma(i) < \sigma(j) \Longleftrightarrow \wpi > \wpj \vee(\wpi = \wpj \text{ and }i < j)$
        \ENSURE $\shat$
    \end{algorithmic}
\end{algorithm}

\section{Technical Results and Proofs for Section~\ref{sec: FR learning}}\label{app: FR learning}
\subsection{The Pairwise-preference Probabilities}
We begin by showing that the pairwise-preference probabilities $\pij$ behave consistently with the underlying utilities. 

\begin{lemma}\label{lem: consistency of PPPs}
    Let $\pij = \pp{U_i > U_j}$ be the pairwise-preference probability defined in Section \ref{sec: FR learning}. Then, $\pij > 1/2$  if and only if $u_i > u_j$.
\end{lemma}
\begin{proof}
    We denote by $f$ and $F$ the density and the CDF of $\nu$, respectively. Additionally, we let $(a,b)$ be an interval fully contained in $\supp{f}$.
    
    If $u_i = u_j$, then, by symmetry, we have $\pij = 1/2$. If $u_i > u_j$, then 
    \begin{equation*}
        \begin{split}
            \pij & = \int_\rr \pp{U_j < U_i\, | \, U_i = x}f(x-u_i)\, \mathrm{d}x \\
            & = \int_\rr F(x - u_j) f(x - u_i)\, \mathrm{d}x \\
            & \stackrel{\text{(i)}}{=} \int_\rr F(y + u_i - u_j)f(y)\, \mathrm{d}y,
        \end{split}
    \end{equation*}
    where we used the change of variable $y = x - u_i$ for (i). Now consider the function $g:[0,\infty)\rightarrow \rr$ defined by 
    \[
    g(t) := \int_\rr F(y + t)f(y)\, \mathrm{d}y.
    \]
    Taking the first derivative, we find 
    \[
    g'(t) = \int_\rr f(y+t)f(y) \, \mathrm{d}y.  
    \]
    Since $(a,b)\subset\supp{f}$, for $t < b-a$, there exists a subinterval of $\supp{f}$ where both $f(y+t)$ and $f(y)$ are strictly positive. Thus, 
    \[
    g'(t) > 0 \quad\text{for every }t \in [0,b-a).
    \]
    Moreover, $g'(t)\geq 0$ for every $t \geq 0$, so $g$ is non-decreasing on the whole half-line. Together with the fact that $g(0) = 1/2$, this implies $g(t) > 1/2$ for every $t\in (0,b-a)$ and, by monotonicity, $g(t) \geq g(t_0) > 1/2$ for every $t \geq b-a$ and any fixed $t_0 \in (0,b-a)$. Since $\pij = g(u_i - u_j)$ and $u_i - u_j > 0$, we conclude that $\pij > 1/2$. 

    To prove the converse implication, suppose $\pij > 1/2$. As noted above, we cannot have $u_i = u_j$, so either $u_i < u_j$ or $u_i > u_j$. The former would imply $\pji > 1/2$, but since $\pij + \pji = 1$, this yields a contradiction. 
\end{proof}

\subsection{Concentration Bounds}\label{app: concentration bounds FR}

The following result provides anytime-concentration bounds for the unbiased estimators introduced in Section \ref{sec: FR learning}. Our inequalities are based on \citep{maurer_empiricalbernstein}. 

\begin{lemma}\label{lem: concentrations for FR}
    Fix $\delta\in (0,1)$ and let $\deltan{n}$ be the time-$n$ confidence parameter defined in Section~\ref{sec: PP probabilities}. Then, with probability at least $1 - \delta$, the following hold simultaneously for every pair $(i,j)$ with $i\neq j$ and every $n\in\nn\setminus\{1\}$: 
    \begin{align}
        |\pij - \estij{ij}{n}| & \leq \sqrt{\frac{2\svij{ij}{n} \log(4/\deltan{n})}{n}} + \frac{7\log(4/\deltan{n})}{3(n-1)}, \label{pp concentration for FR}\\
        \left|\sqrt{\svij{ij}{n}} - \sqrt{\pij(1-\pij)}\right| & \leq \sqrt{\frac{2\log(2/\deltan{n})}{n-1}}. \label{variance concentration for FR}
    \end{align}
\end{lemma}
\begin{proof}
    For a fixed pair $(i,j)$ with $1\leq i < j\leq k$ and $n\geq 2$, we define the events
    \[
    A_{ij}^n := \left\{|\pij - \estij{ij}{n}| \leq \sqrt{\frac{2\svij{ij}{n}\log(4/\deltan{n})}{n}} + \frac{7\log(4/\deltan{n})}{3(n-1)}\right\}
    \]
    and 
    \[
    B_{ij}^n := \left\{\left|\sqrt{\svij{ij}{n}} - \sqrt{\pij(1-\pij)}\right| \leq \sqrt{\frac{2\log(2/\deltan{n})}{n-1}}\right\}.
    \]
    By Theorem 4 and Theorem 10 in \citep{maurer_empiricalbernstein}, we have $\pp{A_{ij}^n} \geq 1 -\deltan{n}$ and $\pp{B_{ij}^n} \geq 1 - \deltan{n}$. By taking a union bound, we obtain
    \begin{equation*}
        \begin{split}
            \pp{\bigcap_{n\in\nn}\left(\bigcap_{1\leq i< j\leq k} (A_{ij}^n \cap B_{ij}^n)\right)} & = 1 - \pp{\bigcup_{n\in\nn}\left(\bigcup_{1\leq i< j\leq k} (A_{ij}^n)^c \cup (B_{ij}^n)^c\right)} \\
            & \geq 1 - \frac{k(k-1)}{2}\sum_{n=1}^\infty 2 \deltan{n} \\
            & = 1 - \frac{6\delta}{\pi^2}\sum_{n=1}^\infty \frac{1}{n^2} \\
            & \stackrel{\text{(i)}}{=} 1 - \delta,
        \end{split}
    \end{equation*}
    where (i) follows from the well-known fact that $\sum_{n\in\nn}1/n^2 = \pi^2/6$. To conclude, it suffices to observe that, for any two distinct items $i$ and $j$, one has
    \[
    |\pij - \estij{ij}{n}| = |P_{ji} - \estij{ji}{n}|\quad\text{and}\quad \svij{ij}{n} = \svij{ji}{n}.
    \]
\end{proof}

\subsection{Uniform Lower Bound on the Resolution Function}\label{app: lower bound on eta FR}

In this section, we formally derive a uniform lower bound on the FR resolution function that holds across all noise distributions satisfying Assumption \ref{assump: bounded variance}. 

\begin{lemma}\label{lem: lower bound on eta FR}
    Let $\nu\in\mathcal{N}$ be a noise distribution satisfying the bounded-variance condition \emph{\ref{assump: bounded variance}}. Then
    \[
    \eta(t) \geq \eta^*(t)\quad\text{for every }t\in [0,\infty), 
    \]
    where ${{\etastar}}$ is defined by \eqref{eta star for FR}.
\end{lemma}
\begin{proof}
Let $X$ and $Y$ be two i.i.d. random variables with distribution $\nu$, and let $Z:= Y - X$ denote their difference, which is also log-concave. Additionally, $Z$ is centered and symmetric, and its variance satisfies $\var{Z} = \var{X} + \var{X} \leq 2 V$. Let us denote by $\sigma > 0$ the standard deviation of $Z$, and let $H$ be its CDF. Then, for any $t > 0$, we have:
\begin{equation*}
    \begin{split}
        \int_\rr F(x+t)f(x) \, dx & = \int_\rr \pp{Y < x - t}f(x)\, dx = \\
        & = \int_\rr \pp{Y < X - t\, |\, X = x}f(x)\, dx = \\
        & = \pp{Y-X < - t} = \\
        & = \pp{Z < -t} = H(-t). 
    \end{split}
\end{equation*}
The left-hand side is equal to $\xi(t)$, where $\xi$ is the function defined by Equation~\eqref{xi for FR 2}, which satisfies $\eta(t) = 2\xi(t) - 1$. By symmetry, we have: 
\[
\xi(t) = 1 - H(-t), \quad\forall t > 0. 
\]
using Lemma~\ref{lem: Gauss deviation bounds for unimodal r.v.s}, we can bound the term $H(-t)$ as follows: 
\[
    H(-t) \leq \begin{cases}
            \displaystyle \frac{1}{2} - \frac{t}{2\sqrt{3}\sigma} & \text{if }0\leq t < 2\sigma/\sqrt{3}\\
        \displaystyle \frac{2\sigma^2}{9t^2} & \text{if }t \geq 2\sigma/\sqrt{3}
    \end{cases}.
\]
We now consider two cases. 

\textit{Case 1}. If $t\in (0, 2\sigma\sqrt{3})$, then $H(-t)\leq 1/2 - t/2\sqrt{3}\sigma$. Hence, we obtain: 
\[
\begin{split}
    1 - H(-t) & \geq \frac{1}{2} + \frac{t}{2\sqrt{3}\sigma} \stackrel{\sigma\leq \sqrt{2V}}{\geq} \\
    & \geq \frac{1}{2} + \frac{t}{2\sqrt{6V}} = \\
    & = \frac{1}{2}\left( 1 + \frac{t}{\sqrt{6V}}\right) . 
\end{split}
\]
Using the relation $\eta(t) = 2\xi(t) - 1$, we conclude that $\eta(t) \geq t/\sqrt{6V}$. 

\textit{Case 2}. If $t\geq 2\sigma/\sqrt{3}$, we have 
\[
H(-t) \leq \frac{2\sigma^2}{9t^2}\leq \frac{2\sigma^2}{9(2\sigma/\sqrt{3})^2} = \frac{1}{6}.
\]
This implies 
\[
\xi(t) \geq \frac{1}{2}\left(1 + \frac{2}{3}\right),
\]
from which we obtain that $\eta(t) \geq 2/3$. 

Combining the conclusions of both cases yields the desired bound. 
\end{proof}

\subsection{The PAC Property}\label{app: pac property FR}

In this section, we prove two results showing that the FR-VUR algorithm presented in Section \ref{sec: FR learning} is $\epdel$-PAC. 

\begin{lemma}\label{lem: the output graph is a DAG}
    With probability at least $1-\delta$, the pairwise-preference graph $G_T$ is a subgraph of the utility graph.   
\end{lemma}
\begin{proof}
    With probability at least $1-\delta$, the inequality given by \eqref{pp concentration for FR} holds for all $i\neq j$ and $n\geq 2$. In particular, taking $n=T$, this further implies
    \[
    |\estij{ij}{T} - \pij| < \frac{{{\etastar}}}{4} \quad\forall i\neq j. 
    \]
    If $\dedge{i}{j}\in \mathcal{E}_T$, i.e., there is an edge from $i$ to $j$ in $G_T$, then 
    \[
    \pij = \estij{ij}{T} + (\pij - \estij{ij}{T}) > \frac{1}{2} + \frac{{{\etastar}}}{4} - \frac{{{\etastar}}}{4} = \frac{1}{2}. 
    \]
    Since $\pij > 1/2$ if and only if $u_i > u_j$ (see Lemma \ref{lem: consistency of PPPs}), we conclude that $\dedge{i}{j}$ is an edge of $G_u$. 
\end{proof}

\begin{lemma}\label{lem: PAC guarantee for FR}
    With probability at least $1-\delta$, if $i$ and $j$ are any two items such that $u_i - u_j \geq \varepsilon$, then $\shat(i) < \shat(j)$. 
\end{lemma}
\begin{proof}
    With probability at least $1 - \delta$, inequality \eqref{pp concentration for FR} holds for all $i\neq j$ and $n\geq 2$. In this regime, the condition of Lemma \ref{lem: the output graph is a DAG} is satisfied, so, taking $n = T$, 
    \[
    |\estij{ij}{T} - \pij| < \frac{{{\etastar}}}{4} \quad\forall i\neq j. 
    \]
    Now let $i$ and $j$ be any two items such that $|u_i - u_j| \geq \varepsilon$, say $u_i \geq u_j + \varepsilon$. If $\eta(\varepsilon)$ denotes the value of the resolution function computed for the true noise distribution, then Lemma \ref{lem: lower bound on eta FR} ensures $\eta(\varepsilon) \geq {{\etastar}}$. By \eqref{xi for FR} and the monotonicity of $\xi$, $\pij \geq \xi(\varepsilon)$. Using the definition of $\eta(\varepsilon)$, we have
    \[
    \pij \geq \frac{1}{2} + \frac{\eta(\varepsilon)}{2} \geq \frac{1}{2} + \frac{{{\etastar}}}{2}. 
    \]
    Hence, 
    \[
    \estij{ij}{T} = \pij + (\estij{ij}{T} - \pij) > \frac{1}{2} + \frac{{{\etastar}}}{2} - \frac{{{\etastar}}}{4} = \frac{1}{2} + \frac{{{\etastar}}}{4}.
    \]
    This implies $\dedge{i}{j}\in \mathcal{E}_T$. Since $G_T$ is a DAG, from the definition of $\shat$, it follows that $\shat(i) < \shat(j)$. 
\end{proof}

\subsection{Sample Complexity Upper Bound}\label{app: proof of sample complexity upper bound FR}

In this section, we prove Theorem \ref{thm: sample complexity upper bound FR}. To do so, we first define $\mathcal{E}$ as the set of outcomes for which the following hold simultaneously for every $n\in\nn\setminus\{1\}$ and any two distinct $i$ and $j$:
\begin{align}
    |\estij{ij}{n} - \pij| & \leq \sqrt{\frac{2\svij{ij}{n} \log(4/\deltan{n})}{n}} + \frac{7\log(4/\deltan{n})}{3(n-1)}, \label{condition 1 for E FR}\\
    \left|\sqrt{\svij{ij}{n}} - \sqrt{\pij(1-\pij)}\right| & \leq \sqrt{\frac{2\log(2/\deltan{n})}{n-1}}, \label{condition 2 for E FR}
\end{align}
where $\deltan{n}$ is the one given in Section~\ref{sec: PP probabilities}. 

If \eqref{condition 2 for E FR} holds, then multiplying both sides of $\sqrt{\svij{ij}{n}} \leq \sqrt{\pij(1-\pij)} + \sqrt{2\log(2/\deltan{n})/(n-1)}$ by $\sqrt{2\log(4/\deltan{n})/n}$ gives 
\[
\sqrt{\frac{2\svij{ij}{n} \log(4/\deltan{n})}{n}} \leq \sqrt{\frac{2\pij(1-\pij) \log(4/\deltan{n})}{n}} + \frac{2\log(4/\deltan{n})}{n-1}, 
\]
where we used $\log(2/\deltan{n}) \leq \log(4/\deltan{n})$ and $n(n-1)\geq (n-1)^2$. Adding $7\log(4/\deltan{n})/(3(n-1))$ to both sides, we obtain 
\begin{equation}\label{empirical to true variance FR}
\sqrt{\frac{2\svij{ij}{n} \log(4/\deltan{n})}{n}} + \frac{7\log(4/\deltan{n})}{3(n-1)} \leq \sqrt{\frac{2\pij(1-\pij) \log(4/\deltan{n})}{n}} + \frac{13\log(4/\deltan{n})}{3(n-1)}, 
\end{equation}
whose left-hand side is exactly the quantity appearing in the stopping rule \eqref{FR stopping condition}. Thus, if $\mathcal{E}$ occurs, the stopping rule is satisfied if the following hold simultaneously for every $i\neq j$: 
\begin{align}
    \sqrt{\frac{2\pij(1-\pij) \log(4/\deltan{n})}{n}} & < \frac{{{\etastar}}}{8}, \label{sufficient condition for stopping rule FR 1} \\ 
    \frac{13\log(4/\deltan{n})}{3(n-1)} & < \frac{{{\etastar}}}{8}. \label{sufficient condition for stopping rule FR 2}
\end{align}
Condition \eqref{sufficient condition for stopping rule FR 1} holds for every $i\neq j$ if the stronger requirement below is satisfied:
\[
n > \max_{i\neq j}(\pij(1-\pij)) \cdot \frac{128\log(4/\deltan{n})}{{{\etastar}}^2}.
\]
In particular, it holds for $n_0$, which we define as the minimum positive integer $n$ such that 
\[
n \geq \max\left\{2,  p(1-p)\cdot \frac{128\log(4/\deltan{n})}{{{\etastar}}^2}\right\},
\]
where $p\in [0,1]$ is chosen so that  
\[
p(1-p) = \max_{i\neq j}(\pij(1-\pij)).
\]
Using the explicit expression of $\deltan{n}$ given in Section~\ref{sec: PP probabilities}, we have 
\[
n_0 \leq p(1-p)\cdot \frac{128\log(4/\deltan{n_0})}{{{{\etastar}}}^2} + 1 \leq \max\left(p(1-p)\cdot\frac{256}{{{{\etastar}}}^2}\log\left(\frac{2kn_0\pi}{\sqrt{6\delta}}\right), 2\right).
\]
By Lemma \ref{lem: Kaufmann - log bound}, there exists a universal constant $C_0$ such that 
\begin{equation}\label{bound on n zero FR}
n_0 \leq \max\left(\frac{C_0p(1-p)}{{{\etastar}}^2}\log\left(\frac{C_0p(1-p)}{{{\etastar}}^2}\cdot \frac{k}{\delta}\right),2\right),
\end{equation}
where $\log(0)$ is understood as $-\infty$. 

Now we consider the condition given by \eqref{sufficient condition for stopping rule FR 2}, and define $n_1$ as the minimum positive integer $n$ such that 
\[
n \geq \left\lceil \frac{104\log(4/\deltan{n})}{3{{\etastar}}} + 1 \right\rceil. 
\]
Such an integer exists, since the right-hand side grows logarithmically in $n$, while the left-hand side grows linearly. By minimality, the inequality above fails at $n_1 - 1$, and since $\deltan{n}$ is non-increasing in $n$, this yields 
\[
n_1 \leq \frac{104\log(4/\deltan{n_1})}{3{{\etastar}}} + 2. 
\]
It follows that 
\[
n_1 \leq \max\left(\frac{208\log(4/\deltan{n_1})}{3{{\etastar}}}, 4\right).
\]
Using the definition of $\deltan{n}$, Lemma \ref{lem: Kaufmann - log bound} ensures that there exists a constant $C_1 > 0$ such that
\begin{equation}\label{bound on n 1 FR}
n_1 \leq \max\left(\frac{C_1}{{{\etastar}}}\log\left(\frac{C_1}{{{\etastar}}}\cdot\frac{k}{\delta}\right), 4\right).
\end{equation}
If we let $n_2:= n_0 + n_1$, then for every $n\geq n_2$, the conditions given by \eqref{sufficient condition for stopping rule FR 1} and \eqref{sufficient condition for stopping rule FR 2} are both satisfied.

Now, the sample complexity can be decomposed as
\begin{equation}\label{expectation decomposition FR}
\ee{T} = \ee{T\cdot\mathbbm{1}_\mathcal{E}} + \ee{T\cdot\mathbbm{1}_{\mathcal{E}^c}}.
\end{equation}
Using the tail-sum formula for the expectation, we can rewrite 
\[
\ee{T\cdot\mathbbm{1}_{\mathcal{E}^c}} = \sum_{n=1}^\infty \pp{T\cdot \mathbbm{1}_{\mathcal{E}^c}\geq n}.
\]
If $n > n_2$, then necessarily $\mathbbm{1}_{\mathcal{E}^c} = 1$. Indeed, we defined $n_2$ in such a way that, if $\mathcal{E}$ occurs, then the stopping condition is met for all $n\geq n_2 + 1$. Therefore, we have
\[
\sum_{n=n_2+1}^\infty \pp{T\cdot \mathbbm{1}_{\mathcal{E}^c}\geq n} = \sum_{n=n_2+1}^\infty \pp{T\geq n}.
\]
If $T\geq n$, then the stopping condition was not satisfied at time $n-1$. Thus, at least one of \eqref{condition 1 for E FR} and \eqref{condition 2 for E FR} is violated for at least one pair $(i,j)$ with $i\neq j$. By a union bound over the $k(k-1)/2$ unordered pairs and the two conditions, this occurs with probability at most $k(k-1)\deltan{n-1}$. Hence, using the definition of $\deltan{n}$ given in Section~\ref{sec: PP probabilities}, 
\[
\sum_{n=n_2+1}^\infty \pp{T\geq n} \leq k(k-1)\sum_{n=n_2}^\infty \deltan{n} = \sum_{n=n_2}^\infty \frac{6\delta}{\pi^2n^2} \leq C_2,
\]
for a universal constant $C_2>0$. It follows that 
\[
\sum_{n=1}^\infty \pp{T\cdot \mathbbm{1}_{\mathcal{E}^c}\geq n} \leq n_2\cdot \pp{\mathcal{E}^c} + C_2.
\]
Using the decomposition given by \eqref{expectation decomposition FR}, this implies
\[
\ee{T} \leq n_2 + C_2. 
\]
To prove the claim, we use the upper bounds given by \eqref{bound on n zero FR} and \eqref{bound on n 1 FR} together with the definition of ${{\etastar}}$. 
\qed

\subsection{Information-theoretic Lower Bound}\label{app: proof of sample complexity lower bound FR}

We derive an information-theoretic lower bound for the ranking problem in the FR setting, thus proving Theorem \ref{thm: sample complexity lower bound for FR}. To this end, we fix a variance upper bound $V > 0$ and consider the centered Gumbel distribution $\nu$ with scale parameter $\beta = \sqrt{6V}/\pi$ and location parameter $\mu = -\beta\gamma$, where $\gamma$ is the Euler--Mascheroni constant. This ensures that $\nu$ is centered and that its variance is equal to $V$. We fix an accuracy level $\varepsilon$ in the small-gap regime $\varepsilon < \sqrt{6V}/(2\pi)$, which is equivalent to requiring $\Tilde{\varepsilon} := \varepsilon/\beta \in (0,1/2)$. Then, consider the utility vectors $\bm{u}$ and $\bm{v}$ given by 
\[
u_1 = v_2 = \varepsilon 
\]
and
\[
u_\ell = v_m = 0\quad\text{for every }\ell \in [k]\setminus\{1\} \text{ and } m\in [k]\setminus \{2\}.
\]
Since the location parameter shifts all sampled utilities by the same constant, the induced ranking law does not depend on $\mu$, and we can consider $\rum$ and $\rumv$. These are seen as two instances of a multi-armed bandit problem, where the number of arms is $n=1$. At each round $t$, the learner plays the only available arm. The observed reward is the feedback signal, consisting of a full ranking of items $1,\dots,k$. The \textit{ranking distributions} $P$ and $Q$ correspond to the reward distributions, and are given by 
\[
P(\sigma) = \pp{U_{(i)} = U_{\sigma^{-1}(i)}\, \, \forall i\in [k] \, | \, \rum}
\]
and 
\[
Q(\sigma) = \pp{U_{(i)} = U_{\sigma^{-1}(i)}\,\, \forall i\in [k] \, | \, \text{RUM}(v,\nu)},
\]
for every $\sigma\in \Sk$. 

We denote by $\bb{P}_u$ and $\bb{P}_v$ the probability measures on $(\Omega,\mathcal{F})$ induced by $\rum$ and $\rumv$, respectively. The next result provides an upper bound on the KL divergence between the two ranking distributions. 

\begin{lemma}\label{bounds on the KL divergence between reward distributions FR}
    Let $P$ and $Q$ be the ranking distributions introduced above. Then, the KL divergence is upper-bounded as 
    \begin{equation*}
    \kl{P}{Q} \leq \frac{\pi^2\varepsilon^2}{V}.
    \end{equation*}
\end{lemma}
\begin{proof}
    In what follows, we abuse notation and use $P[\cdot]$ and $\bb{P}_u[\cdot]$ to denote the same probability measure.  

    Consider any sample from $\rum$. Let $\sigma\in \Sk$ and denote by $i_j$ the index of the $j$-th order statistic. Then, we can factorize the probability of observing $\sigma$ as 
    \begin{equation}\label{factorization prob of sigma}
        P(\sigma) = \bb{P}_u[i_1= \sigma^{-1}(1)]\cdot \prod_{t=2}^{k-1}\bb{P}_u[i_t = \sigma^{-1}(t)\, | \, i_j = \sigma^{-1}(j)\, \text{ for }j\leq t-1]. 
    \end{equation}
    A similar factorization holds for $Q$, and these allow us to think of the sampling process as a sequential one: first, the top item $i_1$ is chosen among all items $1,\dots,k$; then, item $i_2$ is picked among all the remaining items; and so on, until only one item is left and is chosen as the last element of the ranking. 
    
    According to this interpretation, if $S_t\subseteq [k]$ is the set of items left after the first $t-1$ items have been chosen, we define
    \[
    P[i \, |\, S_t] := \frac{e^{u_{i}/\beta}}{\sum_{\ell\in S_t}e^{u_\ell/\beta}} \quad\text{for all }i\in S_t.
    \]
    Now, using \eqref{factorization prob of sigma}, we can write 
    \begin{equation}\label{chain rule for KL}
        \kl{P}{Q} = \sum_{t=1}^{k-1}\bb{E}_{\nu,\bm{u}}[\kl{P[\cdot\, |\, S_t]}{Q[\cdot\, |\, S_t]}],
    \end{equation}
    where the expectation is taken with respect to $\bb{P}_u$ over all pairs $(i,S_t)$ where necessarily $i\in S_t$. 

    For simplicity, we introduce the parameter $\tilde{\varepsilon} := \varepsilon/\beta$, which allows us to rewrite 
    \[
    u_1/\beta = \Tilde{\varepsilon}\,\text{ and }\, u_\ell/\beta = 0\,\text{ for all }\,\ell > 1, 
    \]
    as well as 
    \[
    v_2/\beta = \Tilde{\varepsilon}\,\text{ and }\, v_\ell/\beta = 0\,\text{ for all }\,\ell \neq 2.
    \]
    Recall that $\Tilde{\varepsilon}\in (0,1/2)$, which is exactly the small-gap condition $\varepsilon < \sqrt{6V}/(2\pi)$ assumed above. To estimate the KL divergence between $P$ and $Q$, we consider three different cases. In the following analysis, we fix any $t\in [k-1]$. 

    \textbf{Case 1}: $\{1,2\}\cap S_t = \emptyset$. \\
    In this case, $1$ and $2$ have already been picked as the $t$-th decision round starts, so the conditional winning probabilities given $S_t$ induced by $\bm{u}$ and $\bm{v}$ are the same:
    \[
    \kl{ P[\cdot \, |\, S_t]}{Q[\cdot\, |\, S_t]} = 0.
    \]

    \textbf{Case 2}: $1,2\in S_t$. \\
    In this case, we have
    \begin{equation*}
        \begin{split}
        \kl{ P[\cdot \, |\, S_t]}{Q[\cdot \, |\, S_t]} & = \sum_{i_t\in S_t} P[i_t \, |\, S_t] \log \frac{ P[i_t \, |\, S_t]}{Q[i_t \, |\, S_t]} \\
        & = \sum_{i_t\in S_t\setminus \{1,2\}} P[i_t \, |\, S_t] \log \frac{ P[i_t \, |\, S_t]}{Q[i_t \, |\, S_t]} +  \sum_{\ell=1,2}P[\ell\, |\, S_t]\log \frac{ P[\ell\, |\, S_t]}{Q[\ell\, |\, S_t]} \\
        & \stackrel{\text{(i)}}{=} \frac{e^{\Tilde{\varepsilon}}}{|S_t| - 1 + e^{\Tilde{\varepsilon}}}\log(e^{\Tilde{\varepsilon}}) + \frac{1}{|S_t| - 1 + e^{\Tilde{\varepsilon}}}\log(e^{-\Tilde{\varepsilon}}) \\
        & = \frac{(e^{\Tilde{\varepsilon}}-1)\Tilde{\varepsilon}}{|S_t| - 1 + e^{\Tilde{\varepsilon}}} \\
        & \stackrel{\text{(ii)}}{\leq} \frac{\Tilde{\varepsilon}({\Tilde{\varepsilon}} + {\Tilde{\varepsilon}}^2)}{|S_t|},
        \end{split}
    \end{equation*}
    where (i) follows from the fact that the first sum is equal to zero and (ii) is due to $e^x - 1 \leq x + x^2$ for $x\in (0,1/2)$. 

    \textbf{Case 3}: $|S_t \cap \{1,2\}| = 1$.\\
    In this case, we consider two further subcases corresponding to: (a) $1\in S_t$ and $2\notin S_t$, and (b) $2\in S_t$ and $1\notin S_t$. 
    
    Consider (a) first. Setting $W := \sum_{j\in S_t}e^{u_j/\beta}$ and $W' := \sum_{j\in S_t}e^{v_j/\beta}$, we obtain
    \begin{equation*}
        \begin{split}
            \kl{ P[\cdot \, |\, S_t]}{Q[\cdot \, |\, S_t]} &=  \sum_{j\in S_t} P[i_t = j\, |\, S_t]\log \frac{ P[i_t = j\, |\, S_t]}{Q[i_t=j\, |\, S_t]} \\
            & = \frac{e^{\Tilde{\varepsilon}}}{W}\log\frac{W'\cdot e^{\Tilde{\varepsilon}}}{W} + \sum_{j\in S_t\setminus\{1\}}\frac{1}{W}\log \frac{W'}{W} \\
            & = \frac{e^{\Tilde{\varepsilon}}}{W}\left(\Tilde{\varepsilon} + \log\frac{W'}{W}\right) + \sum_{j\in S_t\setminus\{1\}}\frac{1}{W}\log\frac{W'}{W} \\
            & = \frac{e^{\Tilde{\varepsilon}}\Tilde{\varepsilon}}{W} + \log\frac{W'}{W}. 
        \end{split}
    \end{equation*}
    To bound the term with the logarithm, observe that $\log(1/y)\leq 1/y -1$ for every $y>0$. Thus, we have
    \[
    \log\frac{W'}{W} = \log\frac{|S_t|}{W} \leq \frac{|S_t|}{W} - 1 = \frac{1-e^{\Tilde{\varepsilon}}}{W},
    \]
    where we used the fact that $W'=|S_t|$ and $W = |S_t|-1+e^{\Tilde{\varepsilon}}$. It follows that 
    \begin{equation*}
        \begin{split}
            \kl{ P[\cdot \, |\, S_t]}{Q[\cdot \, |\, S_t]} & \leq \frac{e^{\Tilde{\varepsilon}}\Tilde{\varepsilon}}{W} + \frac{1-e^{\Tilde{\varepsilon}}}{W} \\
            & = \frac{e^{\Tilde{\varepsilon}}\Tilde{\varepsilon} + 1 - e^{\Tilde{\varepsilon}}}{W} \\
            & \leq \frac{{\Tilde{\varepsilon}}^2}{|S_t|},
        \end{split}
    \end{equation*}
    using the inequality $W \geq |S_t|$ and observing that 
    \[
    e^x x + 1 - e^x = e^x(x-1) + 1 \leq x^2\, \text{ for }x\in (0,1/2). 
    \]
    
    The calculations for scenario (b) are similar and yield
    
    \begin{equation}\label{case 3.2}
    \begin{split}
        \kl{ P[\cdot \, |\, S_t]}{Q[\cdot \, |\, S_t]} & = \log\frac{W'}{W} - \frac{\Tilde{\varepsilon}}{W} \\
        & = \log\left(1 + \frac{e^{\Tilde{\varepsilon}} - 1}{|S_t|}\right) - \frac{\Tilde{\varepsilon}}{|S_t|} \\
        & \stackrel{\text{(i)}}{\leq}  \frac{e^{\Tilde{\varepsilon}} - 1}{|S_t|} -\frac{\Tilde{\varepsilon}}{|S_t|} \\
        & =  \frac{e^{\Tilde{\varepsilon}} - 1 - \Tilde{\varepsilon}}{|S_t|} \\
        & \stackrel{\text{(ii)}}{\leq}  \frac{\Tilde{\varepsilon}^2 + \Tilde{\varepsilon} -\Tilde{\varepsilon}}{|S_t|} \\
        & =  \frac{\Tilde{\varepsilon}^2}{|S_t|},
    \end{split}
    \end{equation}
    where (i) is due to $\log(1+x) \leq x$ for $x>0$, whereas for (ii), we relied again on the fact that $e^{\Tilde{\varepsilon}}-1 \leq \Tilde{\varepsilon} + {\Tilde{\varepsilon}}^2$.  
    
    Now, observe that $\bb{E}_{\nu,\bm{u}}\left[\kl{ P[\cdot \, |\, S_t]}{Q[\cdot \, |\, S_t]}\right]$ can be computed by averaging the KL divergence inside the expectation over all choices of $S_t$, where weights are induced by $P$. Using the bounds derived in the three cases above and the fact that $\Tilde{\varepsilon}^3 \leq \Tilde{\varepsilon}^2$, we obtain 
    \begin{equation}\label{bound for KL at round t}
        \begin{split}
            \bb{E}_{\nu,\bm{u}}\left[\kl{ P[\cdot \, |\, S_t]}{Q[\cdot \, |\, S_t]}\right] & \leq P\left[{1,2 \in S_t}\right]\frac{\Tilde{\varepsilon}^2 + \Tilde{\varepsilon}^3}{|S_t|} + P\left[|\{1,2\}\cap S_t| = 1\right]\frac{\Tilde{\varepsilon}^2}{|S_t|} \\
            & \leq \left(1 - P[1,2\notin S_t]\right)\frac{2\Tilde{\varepsilon}^2}{|S_t|},
        \end{split}
    \end{equation}
    where, with abuse of notation, $P(E)$ denotes the probability of event $E$ under the PL model associated with $\bm{u}$. 
    
    Next, let us observe that 
    \begin{equation*}
        \begin{split}
             P[1\in S_t] & = \prod_{\ell=1}^{t-1}\frac{k-\ell}{k-\ell + e^{\Tilde{\varepsilon}}},
        \end{split}
    \end{equation*}
    where $k-\ell+e^{\Tilde{\varepsilon}}\geq k-\ell+1$ for every $\ell \in [t-1]$. This implies 
    \begin{equation*}
        \begin{split}
             P[1\in S_t] \leq \prod_{\ell=1}^{t-1}\frac{k-\ell}{k-\ell+1} = \frac{k-t+1}{k} = \frac{|S_t|}{k}. 
        \end{split}
    \end{equation*}
    On the other hand, one has 
    \begin{equation*}
    \begin{split}
         P[2\in S_t] \leq \prod_{\ell=1}^{t-1}\frac{k-(\ell+1)+e^{\Tilde{\varepsilon}}}{k-\ell + e^{\Tilde{\varepsilon}}} = \frac{|S_t|+e^{\Tilde{\varepsilon}} - 1}{k+e^{\Tilde{\varepsilon}}-1}.  
    \end{split}
    \end{equation*}
    By using a union bound, we have
    \[
    (1 - P\left[1,2\notin S_t\right]) =  P[1\in S_t \text{ or }2\in S_t] \leq \frac{|S_t|}{k} + \frac{|S_t| + e^{\Tilde{\varepsilon}} -1}{k+e^{\Tilde{\varepsilon}} - 1}.
    \]
    Since $\Tilde{\varepsilon} > 0$, $k+ e^{\Tilde{\varepsilon}} - 1 > k$; moreover, $|S_t| \geq 1$. Therefore, we can bound
    \[
    \frac{|S_t| + e^{\Tilde{\varepsilon}} -1}{k+e^{\Tilde{\varepsilon}} - 1} \leq \frac{|S_t|(1 + e^{\Tilde{\varepsilon}} - 1)}{k} = \frac{e^{\Tilde{\varepsilon}}|S_t|}{k}.
    \]
    From this, we obtain 
    \[
     P[1\in S_t \text{ or }2\in S_t] \leq \frac{|S_t|(1 + e^{\Tilde{\varepsilon}})}{k} \leq \frac{3|S_t|}{k},
    \]
    where we used the fact that $e^x + 1\leq 3$ for every $x\in (0,1/2)$. Together with \eqref{bound for KL at round t}, this allows us to derive our final bound 
    \begin{equation*}
        \bb{E}_{\nu,\bm{u}}\left[\kl{ P[\cdot \, |\, S_t]}{Q[\cdot \, |\, S_t]}\right] \leq \frac{6 \Tilde{\varepsilon}^2}{k}.
    \end{equation*}
    
    Combining this with expression \eqref{chain rule for KL} yields 
    \begin{equation*}
    \begin{split}
        \kl{P}{Q} & \leq \sum_{t=1}^{k-1}\frac{6\Tilde{\varepsilon}^2}{k} \\
        & = \frac{6(k-1)\Tilde{\varepsilon}^2}{k} \\
        & \leq 6\Tilde{\varepsilon}^2. 
    \end{split}
    \end{equation*}
    By recalling that $\Tilde{\varepsilon} = \varepsilon/\beta$, we finally get
    \[
    \kl{P}{Q} \leq \frac{6\varepsilon^2}{\beta^2} = \frac{\pi^2 \varepsilon^2}{V}. 
    \]
\end{proof}

Let $\mathcal{A}^* = (\tstar,\shatstar)$ be an $\epdel$-PAC algorithm, and define the $\mathcal{F}_{\tstar}$-measurable event 
\[
E = \left\{\shatstar(1) < \shatstar(2)\right\}. 
\]
Then 
\[
\bb{P}_u[E] \geq 1-\delta\quad\text{and}\quad \bb{P}_v[E] \leq \delta. 
\]
Let us assume that $\delta \in (0,1/4]$, as the confidence parameter is typically chosen to be very close to zero. For $x\geq 1-\delta$ and $y\leq \delta$, the binary relative entropy satisfies  
\[
d(x,y) \geq d(1-\delta,\delta) \geq \log\left(\frac{1}{3\delta}\right). 
\]
Using this, a direct application of Lemma \ref{Lemma 1 - Kaufmann} yields
\[
\bb{E}_{\nu,\bm{u}}[\tstar] \kl{P}{Q} \geq \log\left(\frac{1}{3\delta}\right).
\]
Hence, by Lemma \ref{bounds on the KL divergence between reward distributions FR}, we obtain 
\[
\bb{E}_{\nu,\bm{u}}[\tstar] \geq \frac{V}{\pi^2\varepsilon^2}\log\left(\frac{1}{3\delta}\right). 
\]
In conclusion, since $\log(1/(3\delta)) \geq \frac{1}{5}\log(1/\delta)$ for every $\delta\in (0,1/4]$, we have 
\[
\bb{E}_{\nu,\bm{u}}[\tstar] = \Omega\left(\frac{V}{\varepsilon^2}\log\left(\frac{1}{\delta}\right)\right). 
\]

\section{Technical Results and Proofs for Section~\ref{sec: WO learning}}\label{app: WO learning}
\subsection{Concentration Bounds}\label{app: concentration bounds WO}

The following result provides anytime-concentration bounds for the unbiased estimators introduced in Section \ref{sec: WO learning}. Our inequalities are based on \citep{maurer_empiricalbernstein}. The proof of the lemma below is analogous to the one of Lemma \ref{lem: concentrations for FR}.

\begin{lemma}\label{lem: concentrations for WO}
    Fix $\delta\in (0,1)$ and let $\deltan{n}$ be as defined in Section~\ref{sec: PP probabilities}. Then, with probability at least $1 - \delta$, the following hold simultaneously for all items $i\in [k]$ and every $n\in\nn\setminus\{1\}$:
    \begin{align}
        |\wpi - \esti{i}{n}| & \leq \sqrt{\frac{2\svi{i}{n}\log(4/\deltan{n})}{n}} + \frac{7\log(4/\deltan{n})}{3(n-1)}, \label{concentration of wp}\\
        \left|\sqrt{\svi{i}{n}} - \sqrt{\wpi(1-\wpi)}\right| & \leq \sqrt{\frac{2\log(2/\deltan{n})}{n-1}}. \label{concentration of sv for wp}
    \end{align}
\end{lemma}
\begin{proof}
    For fixed $i\in [k]$ and $n\geq 2$, we define the events
    \[
    A_{i}^n := \left\{|\wpi - \esti{i}{n}| \leq \sqrt{\frac{2\svi{i}{n}\log(4/\deltan{n})}{n}} + \frac{7\log(4/\deltan{n})}{3(n-1)}\right\}
    \]
    and 
    \[
    B_{i}^n := \left\{\left|\sqrt{\svi{i}{n}} - \sqrt{\wpi(1-\wpi)}\right| \leq \sqrt{\frac{2\log(2/\deltan{n})}{n-1}}\right\}.
    \]
    By Theorem 4 and Theorem 10 in \citep{maurer_empiricalbernstein}, we have $\pp{A_{i}^n} \geq 1 -\deltan{n}$ and $\pp{B_{i}^n} \geq 1 - \deltan{n}$. By taking a union bound, we obtain
    \begin{equation*}
        \begin{split}
            \pp{\bigcap_{n\in\nn}\left(\bigcap_{i\in [k]} (A_{i}^n \cap B_{i}^n)\right)} & = 1 - \pp{\bigcup_{n\in\nn}\left(\bigcup_{i\in [k]} (A_{i}^n)^c \cup (B_{i}^n)^c\right)} \\
            & \geq 1 - k\sum_{n=1}^\infty 2 \deltan{n} \\
            & = 1 - \frac{6\delta}{\pi^2}\sum_{n=1}^\infty \frac{1}{n^2} \\
            & \stackrel{\text{(i)}}{=} 1 - \delta,
        \end{split}
    \end{equation*}
    where (i) follows from the well-known fact that $\sum_{n\in\nn}1/n^2 = \pi^2/6$. 
\end{proof}

\subsection{Uniform Lower Bound on the Winner-Only Resolution Factor}\label{app: uniform bound on eta WO}

We fix $\nu\in\mathcal{N}_V$ and denote its density and CDF by $f$ and $F$, respectively, and let its standard deviation be $\sigma\in (0,\sqrt V]$. Then, using the notation introduced in Section~\ref{sec: WO learning}, we let $a:=\inf\supp{f}$ and $b:=\sup\supp{f}$, allowing infinite endpoints. By log-concavity, the interior of the support of $f$ is the open interval $(a,b)$.

\begin{lemma}\label{lem: log-concave density lower bound}
    For almost every $x\in\rr$,
    \begin{equation}\label{eq: log-concave density lower bound}
        f(x)\geq\frac{\min\{F(x),1-F(x)\}}{\sqrt{3}\sigma}
        \geq\frac{F(x)(1-F(x))}{\sqrt{3}\sigma}.
    \end{equation}
\end{lemma}
\begin{proof}
    Let $z_0$ be the median of $\nu$, so that $F(z_0)=1/2$. The hazard rate $f/(1-F)$ is non-decreasing and the reversed hazard rate $f/F$ is non-increasing on $(a,b)$ by log-concavity \citep{bagnoli2005log}. Therefore,
    \[
    \frac{f(x)}{F(x)}\geq 2f(z_0)\quad(x\leq z_0),
    \qquad
    \frac{f(x)}{1-F(x)}\geq 2f(z_0)\quad(x\geq z_0).
    \]
    The median-density bound $f(z_0)\geq 1/(\sqrt{12}\sigma)$ follows from \citet[Proposition~B.2, Equation~(B.6)]{bobkov_ledoux_empirical}. On the two sides of the median, the relevant denominator is respectively $F(x)$ and $1-F(x)$, which proves the first inequality. The second follows from $\min\{q,1-q\}\geq q(1-q)$ for $q\in[0,1]$. Outside $(a,b)$ the bounds hold almost everywhere because their right-hand sides vanish.
\end{proof}

The next lemma compares the winning probabilities of any two items and provide a lower bound on their difference in terms of the largest of the two.

\begin{lemma}\label{lem: winning-probability difference lower bound}
    Let $\nu\in\mathcal{N}_V$ have standard deviation $\sigma$, and let $\bm{u}\in\rr^k$. If $u_i>u_j$, then
    \begin{equation}\label{eq: wp difference lower bound}
        P_i-P_j\geq\frac{1-e^{-(u_i-u_j)/(\sqrt{3}\sigma)}}{2}P_i.
    \end{equation}
\end{lemma}
\begin{proof}
    Set $t:=u_i-u_j>0$, $\Bar{F}:=1-F$, and let $z_0$ be the median of $\nu$. Define
    \[
    D_t(z):=f(z)-f(z+t)\frac{F(z)}{F(z+t)},
    \qquad
    c_t:=\frac{1-e^{-t/(\sqrt{3}\sigma)}}{2},
    \]
    assigning the ratio $F(z)/F(z+t)$ the value zero when its denominator is zero. In that case $F(z)=0$ and $f(z)=0$ almost everywhere as well. Since $F$ is log-concave, $f/F$ is non-increasing on $(a,b)$. Thus $D_t(z)\geq0$ almost everywhere in the support, including when $z+t$ is beyond its upper endpoint. Outside the support, $D_t(z)=0$ almost everywhere. Also, $F(z)\leq F(z+t)$ implies that the following holds for almost every $z\in \rr$: 
    \[
    D_t(z)\geq f(z)-f(z+t).
    \]
    We now prove that, for every $\alpha\in\rr$: 
    \begin{equation}\label{eq: claim on D_t}
        \int_\alpha^\infty D_t(z)\,dz\geq c_t\Bar{F}(\alpha).
    \end{equation}
    To this end, let us first suppose that $\alpha\geq z_0$. If $\Bar{F}(\alpha)=0$, the claim follows from non-negativity. Otherwise, for almost every $z\geq\alpha$ in the support, Lemma~\ref{lem: log-concave density lower bound} gives
    \[
    \frac{f(z)}{\Bar{F}(z)}\geq\frac{1}{\sqrt{3}\sigma}.
    \]
    If $\alpha+t<b$, integration over $[\alpha,\alpha+t]$ yields
    \[
    \log\frac{\Bar{F}(\alpha+t)}{\Bar{F}(\alpha)}
    =-\int_\alpha^{\alpha+t}\frac{f(z)}{\Bar{F}(z)}\,dz
    \leq-\frac{t}{\sqrt{3}\sigma}.
    \]
    Hence $\Bar{F}(\alpha+t)\leq e^{-t/(\sqrt{3}\sigma)}\Bar{F}(\alpha)$. The same inequality holds when $\alpha+t\geq b$, because then $\Bar{F}(\alpha+t)=0$. Consequently,
    \[
    \begin{aligned}
        \int_\alpha^\infty D_t(z)\,dz
        &\geq\Bar{F}(\alpha)-\Bar{F}(\alpha+t) =\\
        &\geq\bigl(1-e^{-t/(\sqrt{3}\sigma)}\bigr)\Bar{F}(\alpha)
        =2c_t\Bar{F}(\alpha)
        \geq c_t\Bar{F}(\alpha).
    \end{aligned}
    \]
    If $\alpha<z_0$, non-negativity and the preceding bound at $z_0$ give
    \[
    \int_\alpha^\infty D_t(z)\,dz
    \geq\int_{z_0}^\infty D_t(z)\,dz
    \geq 2c_t\Bar{F}(z_0)=c_t\geq c_t\Bar{F}(\alpha).
    \]
    This proves Equation~\eqref{eq: claim on D_t}.

    Let $w$ be any bounded, non-negative, non-decreasing function on $\rr$, and set $M:=\sup_z w(z)$. If we denote $L_y:=\{z:w(z)>y\}$, its layer-cake representation is
    \[
    w(z)=\int_0^M\mathbbm{1}\{z\in L_y\}\,dy.
    \]
    Each $L_y$ is an upper interval, possibly empty or all of $\rr$. Equation~\eqref{eq: claim on D_t} applies to these sets as well, by absolute continuity and monotone convergence. Tonelli's theorem therefore gives
    \[
    \begin{aligned}
        \int_\rr w(z)D_t(z)\,dz
        &=\int_0^M\int_{L_y}D_t(z)\,dz\,dy \geq \\
        &\geq c_t\int_0^M\int_{L_y}f(z)\,dz\,dy
        =c_t\int_\rr w(z)f(z)\,dz.
    \end{aligned}
    \]
    Now choose
    \[
    w(z):=F(z+t)\prod_{\ell\in[k]\setminus\{i,j\}}F(z+u_i-u_\ell).
    \]
    Expressing both winning probabilities relative to $u_i$ gives
    \[
    P_i=\int_\rr w(z)f(z)\,dz,
    \qquad
    P_i-P_j=\int_\rr w(z)D_t(z)\,dz.
    \]
    The second identity remains valid at points where $F(z+t)=0$: the corresponding integrand for $P_j$ also contains $F(z)=0$. Applying the preceding inequality yields $P_i-P_j\geq c_tP_i$.
\end{proof}

For an instance satisfying Assumption~\ref{assump: the support contains a,b WO}, $P_j>0$. Since $\sigma\leq\sqrt V$, the preceding lemma implies
\[
\frac{P_i}{P_j}
\geq\frac{2}{1+e^{-t/(\sqrt{3}\sigma)}}
\geq\frac{2}{1+e^{-t/\sqrt{3V}}}
=\xi(t),
\]
where $t=u_i-u_j$ and $\xi$ is the function given by Equation~\eqref{eq: xi for WO}. Accordingly, the resolution factor in Equation~\eqref{eta for WO} is also defined in terms of $V$, rather than the unknown true variance.

\begin{lemma}\label{lem: uniform lower bound for eta WO}
    Fix $V>0$. The functions $\eta$ and $\etastar$ in Equations~\eqref{eta for WO} and \eqref{eta star for WO} satisfy
    \[
    \eta(t)\geq\eta^\star(t),\qquad t\in(0,b-a).
    \]
\end{lemma}
\begin{proof}
    Set $h(x):=(1-e^{-x})/(3+e^{-x})$ for $x\geq0$. Its derivatives are
    \[
    h'(x)=\frac{4e^{-x}}{(3+e^{-x})^2}>0,
    \qquad
    h''(x)=-\frac{4e^{-x}(3-e^{-x})}{(3+e^{-x})^3}<0.
    \]
    Thus, $h$ is increasing on $[0,\infty)$ and concave on the interval $[0,1]$, with $h(0)=0$. By concavity on $[0,1]$, we have $h(x)\geq xh(1)$, while monotonicity gives $h(x)\geq h(1)$ for $x\geq1$. Since $h(1)=(e-1)/(3e+1)$ and $\eta(t)=h(t/\sqrt{3V})$,
    \[
    \eta(t)\geq\frac{e-1}{3e+1}\min\left\{1,\frac{t}{\sqrt{3V}}\right\}=\eta^\star(t).
    \]
\end{proof}

\subsection{Consistency of the Winning Probabilities}\label{app: monotonicity of WPs}

In the FR case, we showed that the pairwise preference probability $P_{ij}$ is larger than $1/2$ if and only $u_i > u_j$. Here, we rely on an analogous result for the winning probabilities. 

\begin{lemma}\label{lem: monotonicity of WPs}
    Let $\nu\in\mathcal{N}_V$ and $\bm{u}\in\rr^k$ satisfy Assumption~\emph{\ref{assump: the support contains a,b WO}}. Then $P_i>0$ for every $i\in[k]$. Moreover, for any two items $i,j\in[k]$ with $u_i\geq u_j$,
    \[
    P_i>P_j \Longleftrightarrow u_i>u_j.
    \]
\end{lemma}
\begin{proof}
    Let $(a,b)$ be the interior of the support of the density $f$ of $\nu$. Conditioning on the noise of item $i$ gives
    \[
    P_i=\int_\rr f(z)\prod_{\ell\neq i}F(z+u_i-u_\ell)\,dz.
    \]
    Since $\Delta_u<b-a$, there is a nonempty open interval of $z\in(a,b)$ on which $z+u_i-u_\ell>a$ for every $\ell\neq i$. The integrand is strictly positive on that interval, so $P_i>0$.
    If $u_i=u_j$, exchangeability of their independent noises gives $P_i=P_j$. If $u_i>u_j$, Lemma~\ref{lem: winning-probability difference lower bound} gives
    \[
    P_i-P_j\geq \frac{1-e^{-(u_i-u_j)/(\sqrt{3}\sigma)}}{2}P_i>0,
    \]
    where $\sigma^2=\bb{V}\textup{ar}_\nu[N]>0$. This proves the equivalence.
\end{proof}

\subsection{Sample Complexity Upper Bound}\label{app: proof of sample complexity upper bound WO}

In this section, we derive the sample complexity upper bound given by Theorem \ref{thm: sample complexity upper bound for WO} for the WO-VUR algorithm. 

To do so, we let $\deltan{n}$ be defined as in Section~\ref{sec: PP probabilities} and define $\mathcal{E}$ as the event consisting of those outcomes for which the following hold simultaneously for all $i\in [k]$ and every $n\in\nn\setminus\{1\}$:
\begin{align}
    |\wpi - \esti{i}{n}| & \leq \sqrt{\frac{2\svi{i}{n}\log(4/\deltan{n})}{n}} + \frac{7\log(4/\deltan{n})}{3(n-1)}, \label{first condition for E WO}\\
    \left|\sqrt{\svi{i}{n}} - \sqrt{\wpi(1-\wpi)}\right| & \leq \sqrt{\frac{2\log(2/\deltan{n})}{n-1}}. \label{second condition for E WO}
\end{align}

Arguing exactly as in Appendix \ref{app: proof of sample complexity upper bound FR}, condition \eqref{second condition for E WO} allows us to replace the sample variance $\svi{i}{n}$ by the true one, at the price of an extra additive term: 
\begin{equation}\label{empirical to true variance WO}
\sqrt{\frac{2\svi{i}{n}\log(4/\deltan{n})}{n}} + \frac{7\log(4/\deltan{n})}{3(n-1)} \leq \sqrt{\frac{2\wpi(1-\wpi)\log(4/\deltan{n})}{n}} + \frac{13\log(4/\deltan{n})}{3(n-1)}. 
\end{equation}
The left-hand side is the quantity appearing both in \eqref{first condition for E WO} and in the stopping rule \eqref{stopping rule for WO}, so that, whenever $\mathcal{E}$ occurs, both can be controlled through the right-hand side. 

First, we determine a positive integer $n_0$ such that, for every $n\geq n_0$, the following holds for every $i\in [k]$: 
\begin{equation}\label{condition defining n_0 WO}
\sqrt{\frac{2\wpi(1-\wpi)\log(4/\deltan{n})}{n}} + \frac{13\log(4/\deltan{n})}{3(n-1)} \leq \frac{\wpi}{2}.
\end{equation}
By \eqref{empirical to true variance WO}, this guarantees $|\wpi - \esti{i}{n}| \leq \wpi/2$ on $\mathcal{E}$. The inequality is satisfied if the following hold simultaneously:
\[
\sqrt{\frac{2\wpi(1-\wpi)\log(4/\deltan{n})}{n}} < \frac{\wpi}{4}\quad \text{and}\quad\frac{13\log(4/\deltan{n})}{3(n-1)} < \frac{\wpi}{4}.
\]
For these conditions to be satisfied, we can enforce the stronger requirement:
\[
n > \frac{32\log(4/\deltan{n})}{\wpi}.
\]
Since we want this to hold for every $i\in [k]$, let us define $\pmin:= \min_{i\in [k]}\wpi$, and let $n_0$ be the minimum positive integer $n$ such that 
\[
n \geq \max\left\{2,\frac{32\log(4/\deltan{n})}{\pmin}\right\}.
\]
It follows that 
\[
n_0 \leq \frac{32\log(4/\deltan{n_0})}{\pmin} + 1 \leq \frac{128}{\pmin}\log\left(2\pi n_0 \sqrt{\frac{k}{3\delta}}\right). 
\]
By Lemma \ref{lem: Kaufmann - log bound}, there exists a universal constant $C_1 > 0$ such that 
\[
n_0 \leq \frac{C_1}{\pmin} \log\left(\frac{C_1}{\pmin}\cdot\frac{k}{\delta}\right). 
\]

Now assume that $n > n_0$. If $\mathcal{E}$ occurs, then the condition given by \eqref{condition defining n_0 WO} ensures $\wpi < 2\esti{i}{n}$ for every $i\in [k]$, so that $\frac{{\etastar}}{2}\esti{i}{n} > \frac{{\etastar}}{4}\wpi$. In this regime, using \eqref{empirical to true variance WO} once more, the stopping condition of $\mathcal{A}$ is fulfilled if the following hold for every $i\in [k]$: 
\begin{equation}\label{bernstein advantage WO}
\sqrt{\frac{2\wpi(1-\wpi)\log(4/\deltan{n})}{n}} < \frac{\wpi}{8}{\etastar}\quad \text{and}\quad\frac{13\log(4/\deltan{n})}{3(n-1)} < \frac{\wpi}{8}{\etastar}.
\end{equation}
In particular, these hold simultaneously for every $i\in [k]$ if the following stronger requirement is met: 
\[
n \geq\frac{256}{\pmin {\etastar}^2}\log\left(\frac{4}{\deltan{n}}\right).
\]
If we define $n_1$ as the minimum positive integer $n > n_0$ such that 
\[
n \geq \left\lceil \frac{256\log(4/\deltan{n})}{\pmin {\etastar}^2}\right\rceil, 
\]
which exists since the right-hand side grows logarithmically in $n$, then either $n_1 = n_0 + 1$ or the inequality above fails at $n_1 - 1$, in which case $n_1 \leq 256\log(4/\deltan{n_1})/(\pmin {\etastar}^2) + 2$. In both cases, we can apply the same reasoning as above to find 
\[
n_1 \leq \frac{C_2}{\pmin {\etastar}^2}\log\left(\frac{C_2}{\pmin {\etastar}^2}\cdot\frac{k}{\delta}\right),
\]
for some universal constant $C_2 > 0$.

Up to re-tuning $C_1$ and $C_2$, our derivations show that, under event $\mathcal{E}$, the stopping condition is met within a number of steps equal to 
\[
T_0 := \left\lceil\frac{C_2}{\pmin {\etastar}^2}\log\left(\frac{C_2}{\pmin {\etastar}^2}\cdot \frac{k}{\delta}\right)\right\rceil. 
\]
Therefore, we have
\[
\ee{T} = \ee{T\cdot \mathbbm{1}_\mathcal{E}} + \ee{T\cdot \mathbbm{1}_{\mathcal{E}^c}} \leq T_0 \cdot \pp{\mathcal{E}} + \ee{T\cdot \mathbbm{1}_{\mathcal{E}^c}}.
\]
Using the tail-sum formula for the expectation, we can rewrite 
\[
\ee{T\cdot \mathbbm{1}_{\mathcal{E}^c}} = \sum_{n=1}^\infty \pp{T\cdot \mathbbm{1}_{\mathcal{E}^c} \geq n} \leq T_0\cdot \pp{\mathcal{E}^c} + \sum_{n=T_0 + 1}^\infty \pp{T\cdot \mathbbm{1}_{\mathcal{E}^c} \geq n}. 
\]
Now, observe that: 
\[
T\cdot \mathbbm{1}_{\mathcal{E}^c} \geq n \Longleftrightarrow T \geq n \, \text{ and }\, \mathcal{E}^c\text{ occurs}. 
\]
However, we showed above that if $n \geq T_0 + 1$ and $T \geq n$, then necessarily $\mathcal{E}^c$ must have occurred, as otherwise we would have found $T\leq T_0$. Thus, for every $n \geq T_0 + 1$, we have 
\[
\pp{T\cdot \mathbbm{1}_{\mathcal{E}^c} \geq n} = \pp{T \geq n}.
\]
If $T\geq n$, then the stopping rule was not satisfied at time $n-1$. This implies that at least one of the two conditions given by \eqref{first condition for E WO} and \eqref{second condition for E WO} was violated for at least one item $i\in [k]$. By a union bound over the $k$ items and the two conditions, this occurs with probability at most $2k\deltan{n-1}$. Hence, we have
\begin{equation*}
    \begin{split}
        \sum_{n=T_0 + 1}^\infty \pp{T\cdot \mathbbm{1}_{\mathcal{E}^c} \geq n} & \leq 2k\sum_{n=T_0 + 1}^\infty \deltan{n-1} \\
        & = \delta \cdot\sum_{n=T_0}^\infty \frac{6}{\pi^2 n^2} \\
        & \leq C_3,
    \end{split}
\end{equation*}
where $C_3 > 0$ is a universal constant. This implies $\ee{T} \leq T_0 + C_3$. 

In conclusion, there exists a universal constant $C>0$ such that 
\begin{equation*}
    \ee{T} \leq \frac{C}{\pmin {\etastar}^2}\log\left(\frac{C}{\pmin {\etastar}^2}\cdot\frac{k}{\delta}\right).
\end{equation*}
Using the definition of ${\etastar}$, when $\sqrt{3V}\geq\varepsilon$, the inequality above can be rewritten as
\[
\ee{T} \leq \frac{C\cdot V}{\pmin \varepsilon^2}\log\left(\frac{C\cdot V}{\pmin \varepsilon^2}\cdot\frac{k}{\delta}\right).
\]
\qed

\subsection{Information-theoretic Lower Bound}\label{app: proof of sample complexity lower bound WO}

Finally, we prove Theorem \ref{thm: sample complexity lower bound WO}, thus deriving an information-theoretic lower bound for the ranking problem under WO feedback. 

For every $\delta \in (0,1)$, algorithm $\mathcal{A}^*$ recovers an accurate representation of the utility ranking with probability at least $1-\delta$, provided that Assumptions~\ref{assump: bounded variance} and \ref{assump: the support contains a,b WO} are satisfied. In other words, if $\bm{u}$ and $\nu$ are chosen so that these conditions are met, then: 
\[
\pp{\shatstar\text{ is }\varepsilon\text{-accurate for }\sigu \, | \, \rum} \geq 1 -\delta.
\]

We now construct two specific instances of the problem, and we use them to provide a lower bound on the sample complexity of $\mathcal{A}^*$. Throughout, we assume $k\geq 3$. To do so, choose the \textit{minimum-probability level} $p\in (0,1/k)$ and set an accuracy level $\varepsilon\in (0,b-a)$. As we shall see, the latter must be chosen ``sufficiently small'', which, in this case, means $\varepsilon \in (0,\beta)$. However, this restriction does not invalidate our sample complexity result due to its asymptotic nature. 

Then, let $\nu$ be the centered Gumbel distribution with scale and location parameters given by 
\[
\beta = \frac{\sqrt{6V}}{\pi}\quad\text{and}\quad \mu = -\beta\gamma,
\]
where $\gamma$ denotes the Euler--Mascheroni constant. In particular, this ensures that $\nu$ is centered and that its variance is equal to $V$. Since the location parameter shifts all sampled utilities by the same constant, the implied winning probabilities do not depend on $\mu$. 

Now, consider the utility vectors $\bm{u}$ and $\bm{v}$ given by:
\begin{itemize}
    \item $u_1 = \varepsilon/2$, $u_2 = -\varepsilon/2$, $u_3 = \dots = u_{k-1} = 0$, and $u_k = L$.
    \item $v_1 = -\varepsilon/2$, $v_2 = \varepsilon/2$, $v_3 = \dots = v_{k-1} = 0$, and $v_k = L$.
\end{itemize}
Here $k\geq 3$ guarantees that item $k$, which carries the utility $L$, is distinct from items $1$ and $2$; for $k = 3$ the intermediate block is empty, and the two instances reduce to $(\varepsilon/2, -\varepsilon/2, L)$ and $(-\varepsilon/2,\varepsilon/2,L)$.
Up to choosing $\varepsilon$ sufficiently small, the positive constant $L$ is chosen so that the minimum winning probability of both instances is equal to $p$. We let $P$ and $Q$ be the \textit{winning distributions} corresponding to $\bm{u}$ and $\bm{v}$, that is, 
\[
\wpi := \frac{e^{u_i/\beta}}{\sum_{\ell=1}^k e^{u_\ell/\beta}} \quad\text{and}\quad Q_i := \frac{e^{v_i/\beta}}{\sum_{\ell=1}^k e^{v_\ell/\beta}} \quad\forall i\in [k].
\]

\begin{lemma}\label{lem: KL divergence for winning distributions WO}
    Let $P$ and $Q$ be the winning distributions introduced above. Then, the KL divergence between $P$ and $Q$ is upper bounded as
    \[
        \kl{P}{Q} \leq \frac{\varepsilon^2\pi^2}{6V}\cdot C\cdot p,
    \]
    where $C>0$ is a universal constant. 
\end{lemma}
\begin{proof}
    By direct calculations, we obtain
    \[
    \kl{P}{Q} = \frac{\varepsilon}{\beta}(P_1 - P_2).   
    \]
    Up to choosing $\varepsilon \in (0,\beta)$, there exists a universal constant $C > 0$ such that 
    \[
    P_1 - P_2 \leq \frac{\varepsilon}{\beta}\cdot C\cdot P_2.
    \]
    Due to our choice of $L$, $P_2 = \pmin = p$, thus implying 
    \[
    \kl{P}{Q} \leq \left(\frac{\varepsilon}{\beta}\right)^2 \cdot C \cdot p = \frac{\varepsilon^2\pi^2}{6V}\cdot C\cdot p.
    \]
\end{proof}

In order to employ Lemma \ref{Lemma 1 - Kaufmann}, we think of our setting as a trivial multi-armed bandit problem, in which the number of arms is $n=1$. At time $t$, the only available arm is played. The observed reward is the feedback signal, which consists of the winner of the current round. The two bandit models for the problem are $\rum$ and $\rumv$, where $\bm{u}$ and $\bm{v}$ are the ones introduced above. Since the noise distribution is of the Gumbel type, all winning probabilities are non-zero. Thus, the two reward distributions are mutually absolutely continuous, as needed in order to apply Lemma \ref{Lemma 1 - Kaufmann}. 

Now, let $\mathcal{E}$ be the event that item $1$ is ranked before item $2$. Since the output ranking is an $\mathcal{F}_T$-measurable function, $\mathcal{E}\in \mathcal{F}_T$. If we define $\bb{P}_u$ and $\bb{P}_v$ as the probability measures induced by $\rum$ and $\rumv$, respectively, a direct application of Lemma \ref{Lemma 1 - Kaufmann} yields
\[
\bb{E}_{\nu,\bm{u}}[\tstar] \kl{P}{Q} \geq d(\bb{P}_u(\mathcal{E}), \bb{P}_v(\mathcal{E})). 
\]
Let us assume that $\delta\in (0,1/4]$. Algorithm $\mathcal{A}^*$ is $\epdel$-PAC, so 
\[
\bb{P}_u(\mathcal{E}) \geq 1 - \delta \quad\text{and}\quad \bb{P}_v(\mathcal{E}) \leq \delta. 
\]
Due to the properties of the binary relative entropy, we have 
\[
d(\bb{P}_u(\mathcal{E}), \bb{P}_v(\mathcal{E})) \geq d(1-\delta, \delta) \geq \log\left(\frac{1}{3\delta}\right). 
\]
It follows that 
\[
\bb{E}_{\nu,\bm{u}}[\tstar] \geq \frac{1}{\kl{P}{Q}}\log\left(\frac{1}{3\delta}\right).
\]
Owing to Lemma \ref{lem: KL divergence for winning distributions WO}, we conclude that 
\[
\bb{E}_{\nu,\bm{u}}[\tstar] \geq  \frac{6V}{Cp\pi^2\varepsilon^2}\log\left(\frac{1}{3\delta}\right). 
\]
In conclusion, since $\log(1/(3\delta)) \geq \frac{1}{5}\log(1/\delta)$ for every $\delta\in (0,1/4]$, we have shown that 
\[
\bb{E}_{\nu,\bm{u}}[\tstar] = \Omega\left(\frac{V}{p\varepsilon^2}\log\left(\frac{1}{\delta}\right)\right).
\]

\section{Auxiliary Results}\label{app: auxiliary}
\subsection{Change-of-measure Lemma for Multi-armed Bandit Problems}

The next result restates Lemma 1 of \citep{kaufmann_changeofmeasure}. As this is typically used in the context of multi-armed bandit problems, we need to introduce some preliminary notation. The number of arms is any positive integer $n$, and we identify the arm set with $[n]$. At time $t$, let $A_t$ and $Z_t$ be the (possibly random) arm being chosen and the corresponding reward, respectively. Let $\mathcal{F}_t = \sigma(A_1,Z_1,\dots,A_t,Z_t)$ be the sigma-algebra generated by the trajectory followed up to time $t$. For every $i\in [n]$, we define $N_i(t)$ as the number of times that arm $i$ was played up to time $t$. We use the subscripts $\bb{P}_\nu$ and $\bb{E}_\nu$ to refer to the probability measure induced by a given bandit model $\nu$ and the implied expectation functional. 

\begin{lemma}\label{Lemma 1 - Kaufmann}
    Let $\nu = (\nu_i)_{i=1}^n$ and $\nu' = (\nu'_i)_{i=1}^n$ be two bandit models, where $\nu_i$ (resp. $\nu'_i$) is the reward distribution of arm $i\in [n]$. Assume that, for every $i\in [n]$, the probability distributions $\nu_i$ and $\nu'_i$ are equivalent, i.e., they are mutually absolutely continuous. For any almost-surely finite stopping time $\tau$ with respect to the filtration $(\mathcal{F}_t)_{t\in\nn}$, the following holds:
    \[
    \sum_{i=1}^n \bb{E}_\nu[N_i(\tau)] \kl{\nu_i}{\nu'_i} \geq \sup_{E\in \mathcal{F}_\tau} d\left(\bb{P}_\nu[E], \bb{P}_{\nu'}[E]\right),
    \]
    where $d(x,y) := x\log(x/y) + (1-x)\log((1-x)/(1-y))$ is the binary relative entropy, with the convention that $d(0,0) = d(1,1) = 0$. 
\end{lemma}

As we shall motivate later, the setting considered in this lemma encompasses ours by choosing the number of arms equal to one.

\subsection{Technical Result for Bounding the Sample Complexity}

The next result is a technical lemma used to derive the sample complexity upper bounds given by Theorem \ref{thm: sample complexity upper bound FR} and Theorem \ref{thm: sample complexity upper bound for WO}. A more general version can be found as Lemma 12 in \citep{kaufmann_logbound}. The slightly modified version presented below is sufficient for our purposes. 

\begin{lemma}\label{lem: Kaufmann - log bound}
    Let $n\in\nn$ and $a,b  > 0$. Suppose that $n \leq  a \log(bn)$. Then 
    \[
    n\leq 2a\log\left(ab\right).
    \]
\end{lemma}
\begin{proof}
    For every $x > 0$, $\log(x) \leq \sqrt{x}$. Therefore, we have
    \begin{equation*}
        \begin{split}
            n & \leq a \log(bn) \leq a \sqrt{bn} = a \sqrt{b} \sqrt{n},
        \end{split}
    \end{equation*}
    from which we obtain 
    \[
    \sqrt{n} \leq a \sqrt{b}. 
    \]
    Equivalently, $n \leq a^2 b$. Since the logarithm is increasing and $b>0$, we have
    \[
    n\leq a\log(bn) \leq a \log\left(b (a^2 b)\right) = 2a \log(ab).
    \]
\end{proof}

\subsection{Deviation Bound for Symmetric Unimodal Distributions}
The following is a standard result providing useful deviation bounds for symmetric unimodal distributions. The version for standardized random variables can be found as Theorem 7.1 in \citep{ion_sharpinequalities}.

\begin{lemma}\label{lem: Gauss deviation bounds for unimodal r.v.s}
    Let $Z$ be a symmetric unimodal random variable with $\var{Z} = \sigma^2 > 0$. Then 
    \begin{equation}
        \pp{Z \geq t} \leq \begin{cases}
            \displaystyle \frac{1}{2}(1 - \frac{t}{\sqrt{3}\sigma}) & \text{if }0\leq t < 2\sigma/\sqrt{3}\\
            \displaystyle \frac{2\sigma^2}{9t^2} & \text{if }t \geq 2\sigma/\sqrt{3}
        \end{cases}.
    \end{equation}
\end{lemma}

\end{document}